\documentclass{article}

\usepackage{times}
\usepackage{graphicx} 
\usepackage{subfigure} 
\usepackage{natbib}
\usepackage{algorithm}
\usepackage{algorithmic}
\usepackage{hyperref}
\usepackage[accepted]{icml2015} 
\usepackage{amsmath} 
\usepackage{amsthm}
\usepackage{amsfonts}
\usepackage{dsfont}
\usepackage{mathtools}
\usepackage{pgfplots}
\usepackage{rotating}

\newtheorem{Pb}{Problem}
\newtheorem{Def}{Definition}
\newtheorem{Thm}{Theorem}
\newtheorem{Lem}{Lemma}
\newtheorem{corollary}{Corollary}
\newtheorem{remark}{Remark}

\newcommand\Red[1]{\textcolor{red}{#1}}
\renewcommand\Red[1]{#1}

\newcommand\projecturl{\url{https://github.com/wbounliphone/reldep}}

\icmltitlerunning{A low variance consistent test of relative dependency}

\begin{document} 

\twocolumn[
\icmltitle{A low variance consistent test of relative dependency}

\icmlauthor{Wacha Bounliphone}{wacha.bounliphone@centralesupelec.fr}
\icmladdress{CentraleSup\'{e}lec \& Inria, Grande Voie des Vignes, 92295 Ch\^{a}tenay-Malabry, France}
\icmlauthor{Arthur Gretton}{arthur.gretton@gmail.com}
\icmladdress{Gatsby Computational Neuroscience Unit, University College London, United Kingdom}
\icmlauthor{Arthur Tenenhaus}{arthur.tenenhaus@centralesupelec.fr}
\icmladdress{CentraleSup\'{e}lec, 3 rue Joliot-Curie, 91192 Gif-Sur-Yvette, France}
\icmlauthor{Matthew B.\ Blaschko}{matthew.blaschko@inria.fr}
\icmladdress{Inria \& CentraleSup\'{e}lec, Grande Voie des Vignes, 92295 Ch\^{a}tenay-Malabry, France}

\icmlkeywords{kernels methods, relative dependency}

\vskip 0.3in
]

\begin{abstract}
We describe a novel non-parametric statistical hypothesis test of relative dependence between a source variable and two candidate target variables.  Such a test enables us to determine whether one source variable is significantly more dependent on a first target variable or a second.  Dependence is measured via the Hilbert-Schmidt Independence Criterion (HSIC), resulting in a pair of empirical dependence measures (source-target 1, source-target 2).  We test whether the first dependence measure is significantly larger than the second. Modeling the covariance between these HSIC statistics leads to a provably more powerful test than the construction of independent HSIC statistics by sub-sampling.  The resulting test is consistent and unbiased, and (being based on U-statistics) has favorable convergence properties.  The test can be computed in quadratic time, matching the computational complexity of standard empirical HSIC estimators.  The effectiveness of the test is demonstrated on several real-world problems: we identify language groups from a multilingual corpus, and we prove that tumor location is more dependent on gene expression than chromosomal imbalances.  Source code is available for download at \projecturl.

\end{abstract}

\section{Introduction}

Tests of dependence are important tools in statistical analysis, and are widely applied in many data analysis contexts. Classical criteria include Spearman's $\rho$ and Kendall's $\tau$, which can 
detect non-linear monotonic dependencies. 
More recent research on dependence measurement has focused on non-parametric measures of dependence, which apply even when the dependence is nonlinear, or the variables are multivariate or non-euclidean (for instance images, strings, and graphs).
The statistics for such tests are diverse, and include kernel measures of covariance \cite{GreFukTeoSonetal08,ZhaPetJanSch11} and correlation \cite{DauNki98,FukGreSunSch08}, distance covariances (which are instances of kernel tests) \cite{SzeRizBak07,SejSriGreFuk13}, kernel regression tests~\cite{Cortes09, Gunn02}, rankings  \cite{HelHelGor13}, and space partitioning approaches \cite{GreGyo10,ResTesFinGroetal11,KinAtw14}.
\Red{Specialization of such methods to univariate linear dependence can yield similar tests to classical approaches such as~\citet{darlington1968multiple,bring1996geometric}.}

For many  problems in data analysis, however, the question of whether dependence exists is secondary: there may be multiple dependencies, and the question becomes which dependence is the strongest. For instance, in neuroscience, multiple stimuli may be present (e.g.\ visual and audio), and it is of interest to determine which of the two has a stronger influence on brain activity \cite{TroKorLan11}. In 
automated translation 
\cite{PetBraClo12}, it is of interest to determine whether documents in a source language are a significantly better match to those in one target language than to another target language, either as a measure of difficulty of the respective learning tasks, or as a basic tool for comparative linguistics.

We present a statistical test which determines whether two target variables have a significant difference in their dependence on a third, source variable. The dependence between each of the target variables and the source is computed using the Hilbert-Schmidt Independence Criterion  \cite{GreBouSmoSch05,GreFukTeoSonetal08}.\footnote{Dependency can also be tested with the correlation operator.  However, Fukumizu et al., \yrcite{fukumizu2007statistical} show that unlike the covariance operator, the asymptotic distribution of the norm of the correlation operator is unknown, so the construction of a computationally efficient test of relative dependence remains an open problem.}  Care must be taken in analyzing the asymptotic behavior of the test statistics, since the two measures of dependence will themselves be correlated: they are both computed with respect to the same source.  Thus, we derive the {\em joint} asymptotic distribution of both dependencies. 
The derivation of our test utilizes classical results of $U$-statistics~\cite{hoeffding1963probability,Serfling1981,arcones1993limit}. In particular, we make use of results by Hoeffding \yrcite{hoeffding1963probability} and Serfling \yrcite{Serfling1981} to determine the asymptotic joint distributions of the statistics (see Theorem~\ref{thm:joint_asympt_dist}). Consequently, we derive the \emph{lowest} variance unbiased estimator of the test statistic.

We prove our approach to have greater statistical power than constructing two uncorrelated statistics on the same data by subsampling, and testing on these.  In experiments, we are able to successfully test which of two variables is most strongly related to a third, in synthetic examples, in a language group identification task, and in a task for identifying the relative strength of factors for Glioma type in a pediatric patient population.

To our knowledge, there do not exist competing non-parametric tests to determine which of two dependencies is strongest.  One related area is that of multiple regression analysis (e.g. \cite{SenSri11}). In this case a linear model is assumed, and it is determined whether individual inputs have a statistically significant effect on an output variable. The procedure does not address the question of whether the influence of one variable is higher than that of another to a statistically significant degree. The problem of variable selection has also been investigated in the case of nonlinear relations between the inputs and outputs \cite{Cortes09, CorMohRos12,SonSmoGreBedetal12}, however this again does not address which of two variables most strongly influences a third. A less closely related area is that of detecting three-variable interactions \cite{SejGreBer13}, where it is determined whether there exists any factorization of the joint distribution over three variables. This test again does not address the issue of finding which connections are strongest, however.


\section{Definitions and description of HSIC}\label{sec:theory}

We base our underlying notion of dependence on the Hilbert-Schmidt Independence Criterion~\cite{GreBouSmoSch05,GreFukTeoSonetal08, SonSmoGreBedetal12}.  All results in this section except for Problem~\ref{pb:mainpb} can be found in these previous works.
\begin{Def}\Red{~\citep[Definition 1,Lemma 1:  Hilbert-Schmidt Independence Criterion]{GreBouSmoSch05}}

\label{def:HSIC_crosscov}
Let $P_{xy}$ be a Borel probability measure over  over ($\mathcal{X} \times \mathcal{Y}, \Gamma \times \Lambda$)  with $\Gamma$ and $\Lambda$ the respective Borel sets on $\mathcal{X}$ and $\mathcal{Y}$, and $P_{x}$ and $P_y$ the  marginal distributions  on domains $\mathcal{X}$ and $\mathcal{Y}$.
 Given separable RKHSs $\mathcal{F}$ and $\mathcal{G}$, 
\Red{
the Hilbert-Schmidt Independence Criterion (HSIC) is defined as the squared HS-norm of the associated cross-covariance operator $C_{xy}$. 
When the kernels $k$, $l$ are associated uniquely withs respective RKHSs $\mathcal{F}$ and $\mathcal{G}$ and bounded, HSIC can be expressed in terms of expectations of kernel functions 
}
\begin{align}
&HSIC(\mathcal{F}, \mathcal{G}, P_{xy}) := \Vert C_{xy} \Vert ^2_{HS}  \nonumber \\
&= \mathbb{E}_{xx'yy'} \left[ k(x,x')l(y,y') \right] + \mathbb{E}_{xx'} \left[ k(x,x') \right] 
\mathbb{E}_{yy'} \left[ l(y,y') \right] \nonumber \\
&- 2 \mathbb{E}_{xy} \left[ \mathbb{E}_{x'} [k(x,x')] \mathbb{E}_{y'} [l(y,y')] \right]  .
\end{align}
\end{Def}
HSIC determines independence: HSIC = 0 iff $P_{xy} = P_xP_y$ when kernels $k$ and $l$ are characteristic on their respective marginal domains~\cite{Asimplercondition}.

With this choice, the problem we would like to solve is described as follows:
\begin{Pb}\label{pb:mainpb} 
Given separable RKHSs  $\mathcal{F}$, $\mathcal{G}$, and $\mathcal{H}$ with $HSIC(\mathcal{F}, \mathcal{G}, P_{xy}) > 0$ and $HSIC(\mathcal{F}, \mathcal{H}, P_{xz}) > 0$, we test the null hypothesis $\mathcal{H}_0$ : $HSIC(\mathcal{F}, \mathcal{G}, P_{xy}) \leq HSIC(\mathcal{F}, \mathcal{H}, P_{xz})$ versus the alternative hypothesis $\mathcal{H}_1$ : $HSIC(\mathcal{F}, \mathcal{G}, P_{xy}) > HSIC(\mathcal{F}, \mathcal{H}, P_{xz})$ at a given significance level $\alpha$.
\end{Pb}

We now describe the asymptotic behavior of the HSIC for dependent variables.

\begin{Thm}\Red{~\citep[Theorem 2:  Unbiased estimator for $HSIC(\mathcal{F}, \mathcal{G}, P_{xy})$]{SonSmoGreBedetal12}}
\label{thm:Unbiased_HSIC}
\Red{We denote by $\mathcal{S}$ the set of observations $\lbrace (x_1,y_1), ..., (x_m,y_m) \rbrace$ of size $m$ drawn i.i.d.\ from $P_{xy}$.  The unbiased estimator $HSIC_m(\mathcal{F},\mathcal{G},\mathcal{S})$ is given by}
\begin{align}
\label{eq:Unbiased_HSIC}
HSIC_m &(\mathcal{F},\mathcal{G},\mathcal{S}) = \dfrac{1}{m(m-3)} \times \\
& \left[ \operatorname{Tr} ( \tilde{\mathbf{K}} \tilde{\mathbf{L}} )  + \dfrac{\mathds{1}'\tilde{\mathbf{K}} \mathds{1} \mathds{1}' \tilde{\mathbf{L}} \mathds{1}}{(m-1)(m-2)} - \dfrac{2}{m-2} \mathds{1}' \tilde{\mathbf{K}} \tilde{\mathbf{L}} \mathds{1} \right] \nonumber
\end{align}
where $\tilde{\mathbf{K}}$ and $\tilde{\mathbf{L}}$ are related to $\mathbf{K}$ and $\mathbf{L}$ by $\tilde{\mathbf{K}}_{ij} = (1- \delta_{ij})\tilde{\mathbf{K}}_{ij}$ and $\tilde{\mathbf{L}}_{ij} = (1- \delta_{ij})\tilde{\mathbf{L}}_{ij}$.
\end{Thm}

\begin{Thm}\Red{~\citep[Theorem 3:  U-statistic of HSIC]{SonSmoGreBedetal12}\label{thm:Ustat_HSIC}}
This finite sample unbiased estimator of $ HSIC_m^{XY}$ can be written as a U-statistic,
\begin{equation}
HSIC_m^{XY} = (m)_4^{-1} \displaystyle \sum_{(i,j,q,r) \in i^m_4} h_{ijqr}
\label{eq:Ustat_HSIC1}
\end{equation}
where $(m)_4 := \dfrac{m!}{(m-4)!}$, the index set $i^m_4$ denotes the set of all $4-$tuples drawn without replacement from the set $\left\lbrace 1, \dots m \right\rbrace $, and the kernel h of the U-statistic is defined as
\begin{equation}
h_{ijqr} = \dfrac{1}{24} \displaystyle \displaystyle \sum_{(s,t,u,v)}^{(i,j,q,r)} k_{st} (l_{st} + l_{uv} -2 l_{su})
\label{eq:Ustat_HSIC2}
\end{equation}
where the kernels $k$ and $l$  are associated uniquely with respective reproducing kernel Hilbert spaces $\mathcal{F}$ and $\mathcal{G}$.
\end{Thm}
\begin{Thm}\Red{~\citep[Theorem 1: Asymptotic distribution of $HSIC_m$]{GreFukTeoSonetal08}}
\label{thm:Asymp_dist_HSIC}
If $\mathbf{E}[h^2] < \infty$, and source and targets are not independent, then, \Red{under $\mathcal{H}_1$}, as $m \rightarrow \infty$,
\Red{
\begin{equation}
\sqrt{m}(HSIC_m^{XY} - HSIC(\mathcal{F}, \mathcal{G}, P_{xy})) \overset{d}{\longrightarrow} \mathcal{N}(0, \sigma^2_{XY})
\label{eq:Asymp_dist_HSIC}
\end{equation}
where $\sigma^2_{XY} = 16 \left( \mathbb{E}_i \left( \mathbf{E}_{j,q,r} h_{ijqr} \right)^2 - HSIC(\mathcal{F}, \mathcal{G}, P_{xy}))  \right)$
with $\mathbf{E}_{j,q,r} := \mathbf{E}_{S_j,S_q,S_r}$.  Its empirical estimate is $\hat{\sigma}_{XY} = 16\left( R_{XY} - (HSIC_m^{XY})^2 \right)$
}
where $R_{XY} = \dfrac{1}{m} \displaystyle\sum_{\substack{i=1}}^m \left( (m-1)_3^{-1} \sum_{(j,q,r) \in i^m_{3} \backslash \left\lbrace i \right\rbrace } h_{ijqr} \right) ^2 $ and the index set $i^m_3 \backslash \left\lbrace i \right\rbrace$ denotes the set of all $3-$tuples drawn without replacement from the set $\left\lbrace 1, \dots m \right\rbrace \backslash \left\lbrace i \right\rbrace$.
\end{Thm}
%



\section{A test of relative dependence}\label{sec:relativedep}

In this section we calculate two dependent HSIC statistics and derive the joint asymptotic distribution of these dependent quantities, which is used to construct a consistent test for Problem~\ref{pb:mainpb}.
We next construct a simpler consistent test,  by computing two independent HSIC statistics on sample subsets.  
While the simpler strategy is superficially attractive and less effort to implement, we prove the dependent strategy is strictly more powerful.

\subsection{Joint asymptotic distribution of HSIC and test}\label{joint_asympt_dist}

\Red{
  In the present section, we compute each HSIC estimate on the full dataset, and explicitly obtain the correlations between the resulting empirical dependence measurements $HSIC_m^{XY}$ and $HSIC_m^{XZ}$.}
We denote by $\mathcal{S}_1 = (X,Y,Z)$ the joint sample of observations which are drawn \textit{i.i.d.}\ with respective Borel probability measure
$P_{xyz}$ defined on the domain $\mathcal{X} \times \mathcal{Y} \times \mathcal{Z}$.  The kernels $k$, $l$ and $d$ are associated uniquely with respective reproducing kernel Hilbert spaces $\mathcal{F}$, $\mathcal{G}$ and $\mathcal{H}$.  Moreover, $\mathbf{K}$, $\mathbf{L}$ and $\mathbf{D} \in {R}^{m \times m}$ are kernel matrices containing $k_{ij} = k(x_i,x_j)$, $l_{ij} = l(y_i,y_j)$ and $d_{ij} = d(z_i,z_j)$.  Let $HSIC_m^{XY}$ and $HSIC_m^{XZ}$ be respectively the unbiased estimators of $HSIC(\mathcal{F},\mathcal{G},P_{xy})$ and $HSIC(\mathcal{F},\mathcal{H},P_{xz})$, written as a sum of U-statistics with respective kernels $h_{ijqr}$ and $g_{ijqr}$ as described in \eqref{eq:Ustat_HSIC2},
\begin{align}
h_{ijqr} &= \dfrac{1}{24} \displaystyle \displaystyle \sum_{(s,t,u,v)}^{(i,j,q,r)} k_{st} (l_{st} + l_{uv} -2 l_{su}), \nonumber \\
g_{ijqr} &= \dfrac{1}{24} \displaystyle \displaystyle \sum_{(s,t,u,v)}^{(i,j,q,r)} k_{st} (d_{st} + d_{uv} -2 d_{su}).
\label{eq:Ustat_HSIC3}
\end{align}
\begin{Thm}
\label{thm:joint_asympt_dist}
\emph{\textbf{(Joint asymptotic distribution of HSIC)}} If $\mathbb{E}[h^2] < \infty$ and $\mathbb{E}[g^2] < \infty$, then 
\Red{
\begin{align}
\sqrt{m} 
& \left( \begin{pmatrix}
HSIC_m^{XY} \\ 
HSIC_m^{XZ} 
\end{pmatrix}
-
\begin{pmatrix}
HSIC(\mathcal{F},\mathcal{G},P_{xy}) \\ 
HSIC(\mathcal{F},\mathcal{H},P_{xz}) 
\end{pmatrix}  
\right) \nonumber \\
&\overset{d}{\longrightarrow} \mathcal{N} 
\left( 
\begin{pmatrix}
0 \\ 
0
\end{pmatrix},
\begin{pmatrix} 
\sigma_{XY}^2 & \sigma_{XYXZ} \\ 
\sigma_{XYXZ} & \sigma_{XZ}^2 
\end{pmatrix} 
\right), 
\label{eq:joint_HSIC}
\end{align}
where $\sigma^2_{XY}$ and $\sigma^2_{XZ}$ are as in Theorem~\ref{thm:Asymp_dist_HSIC}. 
The empirical estimate of $\sigma_{XYXZ}$ is
}
$\hat{\sigma}_{XYXZ} = \dfrac{16}{m} \left( R_{XYXZ} - HSIC_m^{XY} HSIC_m^{XZ} \right)$, where 
\begin{equation}
R_{XYXZ}= \dfrac{1}{m} \displaystyle\sum_{\substack{i=1}}^m \left( (m-1)_3^{-2} \sum_{(j,q,r) \in i^m_{3} \backslash \left\lbrace i \right\rbrace } h_{ijqr}  g_{ijqr}\right).
\label{eq:jointasymptR}
\end{equation} 
\end{Thm}
\begin{proof} 
Eq. ~\eqref{eq:jointasymptR} is constructed with the definition of variance of a U-statistic as given by Serfling, Ch.\ 5 \yrcite{Serfling1981}, where one variable is fixed. Eq.~\eqref{eq:joint_HSIC} follows from the application of Hoeffding, Theorem 7.1 \yrcite{hoeffding1963probability},  which gives the joint asymptotic distribution of U-statistics.
\end{proof}
Based on the joint asymptotic distribution of HSIC described in Theorem~\ref{thm:joint_asympt_dist}, we can now describe a statistical test to solve Problem~\ref{pb:mainpb}: given a sample $\mathcal{S}_1$ as described in Section~\ref{joint_asympt_dist}, $\mathcal{T}(\mathcal{S}_1) : \left\lbrace  (\mathcal{X} \times \mathcal{Y} \times \mathcal{Z})^m \right\rbrace \rightarrow \left\lbrace 0,1 \right\rbrace  $ is used to test the null hypothesis $\mathcal{H}_0$ : $HSIC(\mathcal{F}, \mathcal{G}, P_{xy}) \leq HSIC(\mathcal{F}, \mathcal{H}, P_{xz})$ versus the alternative hypothesis $\mathcal{H}_1$ : $HSIC(\mathcal{F}, \mathcal{G}, P_{xy}) > HSIC(\mathcal{F}, \mathcal{H}, P_{xz})$ at a given significance level $\alpha$.  This is achieved by 
projecting the distribution to 1D using the statistic $HSIC_m^{XY} - HSIC_m^{XZ}$, 
and determining where the statistic falls relative to
a conservative estimate of the
the $1-\alpha$ quantile of the null.
We now derive this conservative estimate.  
A simple way of achieving this is to
rotate the distribution by $\frac{\pi}{4}$ counter-clockwise about the origin, and to integrate the resulting distribution projected
onto the first axis (cf.\ Fig.~\ref{fig:empirical_HSIC_differentsamplesize}).
 Denote
the asymptotically normal distribution of $\Red{\sqrt{m}}[HSIC_m^{XY} HSIC_m^{XZ}]^T$  as $\mathcal{N}(\mu,\Sigma)$. 
The distribution resulting from rotation and projection is
\begin{align}
\label{eq:AsymptoticDistributionKeepAll}
\mathcal{N} &\left( [Q \mu]_{1}, [Q \Sigma Q^{T}]_{11} \right),
\end{align}
where $Q = \dfrac{\sqrt{2}}{2}\begin{pmatrix}
1 & -1 \\ 
1 & 1
\end{pmatrix}$ is the rotation matrix by $\frac{\pi}{4}$ and
\begin{align}
&[Q \mu]_{1} = \frac{\sqrt{2}}{2} \left( 
\Red{
HSIC(\mathcal{F}, \mathcal{G}, P_{xy}) - HSIC(\mathcal{F}, \mathcal{H}, P_{xz})
} 
\right), \\
&[Q \Sigma Q^{T}]_{11} = \frac{1}{2}(\sigma_{XY}^2 + \sigma_{XZ}^2 - 2 \sigma_{XYXZ}).
\label{eq:var_AsymptoticDistributionKeepAll}
\end{align}
Following the empirical distribution from Eq.~\eqref{eq:AsymptoticDistributionKeepAll},
a test with statistic $HSIC_m^{XY} - HSIC_m^{XZ}$ has p-value
\begin{equation}
p \le 1 - \mathbf{\Phi}\left( \frac{ ( HSIC_m^{XY} - HSIC_m^{XZ} )}{\sqrt{ \sigma^2_{XY} + \sigma^2_{XZ} - 2 \sigma_{XYXZ} }} \right),
\label{eq:pvalue}
\end{equation}
where $\mathbf{\Phi}$ is the CDF of a standard normal distribution, and we have made the most conservative possible assumption that $HSIC(\mathcal{F}, \mathcal{G}, P_{xy}) - HSIC(\mathcal{F}, \mathcal{H}, P_{xz})=0$ under the null (the null also allows for the difference in population dependence measures to be negative).

To implement the test in practice, the variances of $\sigma_{XY}^2,\sigma_{XZ}^2$ and $\sigma_{XYXZ}^2$ may be replaced by their empirical estimates. The test will still be consistent for a large enough sample size, since the estimates will be sufficiently well converged to ensure the test is  calibrated.
Eq.~\eqref{eq:jointasymptR} is expensive to compute na\"{i}vely, because even computing the kernels $h_{ijqr}$ and $g_{ijqr}$ of the $U$-statistic itself is a non trivial task.  Following~\citep[Section 2.5]{SonSmoGreBedetal12}, we first form a vector $\mathbf{h_{XY}}$ with entries corresponding to $\sum_{(j,q,r) \in i^m_{3} \backslash \left\lbrace i \right\rbrace } h_{ijqr}$, and a vector $\mathbf{h_{XZ}}$ with entries corresponding to $\sum_{(j,q,r) \in i^m_{3} \backslash \left\lbrace i \right\rbrace } g_{ijqr} $.  Collecting terms in Eq.~\eqref{eq:Ustat_HSIC2} related to kernel matrices $\tilde{\mathbf{K}}$ and $\tilde{\mathbf{L}}$, $\mathbf{h_{XY}}$  
can be written as
\begin{align}\label{eq:compute_hxy}
\mathbf{h_{XY}} &= (m-2)^2 \left( \tilde{\mathbf{K}} \odot \tilde{\mathbf{L}} \right)  \mathds{1} 
- m (\tilde{\mathbf{K}} \mathds{1}) \odot (\tilde{\mathbf{L}} \mathds{1}) \\
&+ (m-2) \left( (\operatorname{Tr}(\tilde{\mathbf{K}}  \tilde{\mathbf{L}})) \mathds{1} - \tilde{\mathbf{K}} (\tilde{\mathbf{L}} \mathds{1}) - \tilde{\mathbf{L}} (\tilde{\mathbf{K}} \mathds{1}) \right) \nonumber \\
&+ (\mathds{1}^T \tilde{\mathbf{L}} \mathds{1})\tilde{\mathbf{K}} \mathds{1} + (\mathds{1}^T \tilde{\mathbf{K}} \mathds{1})\tilde{\mathbf{L}} \mathds{1} - ((\mathds{1}^T \tilde{\mathbf{K}}) (\tilde{\mathbf{L}} \mathds{1}))\mathds{1} \nonumber
\end{align}
where $\odot$ denotes the Hadamard product. 
Then $R_{XYXZ}$ in Eq.~\eqref{eq:jointasymptR} can be computed as $R_{XYXZ} = (4m)^{-1}(m-1)_3^{-2} \mathbf{h_{XY}}^T \mathbf{h_{XZ}}$.  Using the order of operations implied by the parentheses in Eq.~\eqref{eq:compute_hxy}, the computational cost of the cross covariance term is  $\mathcal{O}(m^2)$.  Combining this with the unbiased estimator of HSIC in Eq. \eqref{eq:Unbiased_HSIC} leads to a final computational complexity of $\mathcal{O}(m^2)$.

In addition to the asymptotic consistency result, we provide a finite sample bound on the deviation between the difference of two population HSIC statistics and the difference of two empirical HSIC estimates.  
\begin{Thm}[Generalization bound on the difference of empirical HSIC statistics]
Assume that $k$, $l$, and $d$ are bounded almost everywhere by 1, and are non-negative. Then for $m>1$ and all $\delta > 0$ with probability at least $1-\delta$, for all $p_{xyz}$, the generalization bound on the difference of empirical HSIC statistics is 
\begin{align}
&| \left\lbrace HSIC(\mathcal{F},\mathcal{G},P_{xy}) - HSIC(\mathcal{F},\mathcal{H},P_{xz})\right\rbrace  \nonumber \\
& \hspace{3cm} - \left\lbrace HSIC_m^{XY} - HSIC_m^{XZ} \right\rbrace | \nonumber \\
& \leq 
2 \left\lbrace \sqrt{\frac{log(6/\delta)}{\alpha^2 m}} + \frac{C}{m} \right\rbrace 
\label{eq:HSICreldepGeneralizationBound}
\end{align}
where $\alpha>0.24$ and $C$ are constants.
\end{Thm}
\begin{proof}
In Gretton et al.,~\yrcite{GreBouSmoSch05} a finite sample bound is given for a single HSIC statistic. Eq.~\eqref{eq:HSICreldepGeneralizationBound} is proved by using a union bound.  
\end{proof}
\begin{corollary}
$HSIC_m^{XY} - HSIC_m^{XZ}$ converges to the population statistic at rate $\mathcal{O}(\sqrt{m})$.
\end{corollary}

\subsection{A simple consistent test via uncorrelated HSICs}\label{subsec:consis_appro}

From the result in Eq.~\eqref{eq:Asymp_dist_HSIC}, a simple, consistent  test of relative dependence can be constructed as follows: split the samples from $P_x$ into two equal sized sets denoted by $X'$ and $X''$, and drop the second half of  the sample pairs with $Y$ 
and the first half of the sample pairs with $Z$. 
We will denote the remaining samples as $Y'$ and $Z''$.  We can now estimate the joint distribution of $\Red{\sqrt{m}}[HSIC_{m/2}^{X'Y'}, HSIC_{m/2}^{X''Z''}]^T$ as 
\begin{equation}
\mathcal{N} \left( 
\Red{
\begin{pmatrix} 
HSIC(\mathcal{F}, \mathcal{G}, P_{xy}) \\ 
HSIC(\mathcal{F}, \mathcal{H}, P_{xz}) \end{pmatrix}
}, 
\begin{pmatrix} 
\sigma_{X'Y'}^{2} & 0 \\ 
0 & \sigma_{X''Z''}^2 
\end{pmatrix} 
\right), 
\end{equation}
which we will write as $\mathcal{N} \left( \mu' , \Sigma' \right)$.  Given this joint distribution, we need to 
determine the distribution over the half space defined by 
\Red{
$HSIC(\mathcal{F}, \mathcal{G}, P_{xy}) < HSIC(\mathcal{F}, \mathcal{H}, P_{xz}).$
}
As in the previous section, we  achieve this by rotating the distribution by $\frac{\pi}{4}$  counter-clockwise about the origin, and integrating the resulting distribution projected onto the first axis (cf.\ Fig.~\ref{fig:empirical_HSIC_differentsamplesize}).  The resulting projection of the rotated distribution onto the primary axis is 
\begin{align}\label{eq:AsymptoticDistributionDropHalf}
\mathcal{N} &\left( \left[ Q \mu' \right]_{1}, \left[Q \Sigma' Q^{T} \right]_{11} \right)
\end{align}
where 
\begin{align}
&[ Q \mu']_{1} = \frac{\sqrt{2}}{2} \left( 
\Red{HSIC(\mathcal{F}, \mathcal{G}, P_{xy}) - HSIC(\mathcal{F}, \mathcal{H}, P_{xz})} \right),  \\
&[Q \Sigma' Q^{T}]_{11} = \frac{1}{2}(\sigma_{X'Y'}^2 + \sigma_{X''Z''}^2).
\label{eq:var_AsymptoticDistributionDropHalf}
\end{align}
From this empirically estimated distribution, it is  straightforward to construct a consistent test (cf. Eq. \eqref{eq:pvalue}).  The power of this test varies inversely with the variance of the distribution in Eq. \eqref{eq:AsymptoticDistributionDropHalf}.  

\subsection{The dependent test is more powerful}\label{subsec:lower_variance}

While discarding half the samples leads to a consistent test,
we might expect some loss of power over the approach in Section \ref{joint_asympt_dist}, due to the increase in variance with lower sample size.
In this section, we prove the Section \ref{joint_asympt_dist} test is more powerful than that
of Section \ref{subsec:consis_appro}, regardless of $P_{xy}$ and $P_{xz}$.

We call the simple and consistent approach in Section~\ref{subsec:consis_appro}, the \emph{independent approach}, and the lower variance approach in Section~\ref{joint_asympt_dist}, the \emph{dependent approach}. The following theorem compares these approaches.
\begin{Thm}
The asymptotic relative efficiency (ARE) of the independent approach relative to the dependent approach is always greater than 1.
\label{thm:powerful_test}
\end{Thm}
\begin{remark} The \textit{asymptotic relative efficiency} (ARE) is defined in e.g.~\citet[Chap.5, Section 1.15.4]{Serfling1981}.  If $m_A$ and $m_B$ are the sample sizes at which tests "perform equivalently" (i.e. have equal power), then the ratio $\frac{m_A}{m_B}$ represents the relative efficiency.  When $m_A$ and $m_B$ tend to $+\infty$ and the ratio $\frac{m_A}{m_B} \rightarrow L$ (at equivalent performance), then the value $L$ represents the asymptotic relative efficiency of procedure B relative to procedure A. This example is relevant to our case since we are comparing two test statistics with different asymptotically Normal distributions.
\end{remark}
The following lemma is used for the proof of  Theorem \ref{thm:powerful_test}.
\begin{Lem}(Lower Variance) \label{lem:lowerVariance} 
The variance of the dependent test statistic is smaller than the variance of the independent test statistic.
\end{Lem}
\begin{proof} From the convergence of moments in the application of the central limit theorem~\citep{bahr1965}, we have that $\sigma_{X'Y'}^2 = 2 \sigma_{XY}^2$.  Then the variance summary in Eq.~\eqref{eq:var_AsymptoticDistributionKeepAll} is $\frac{1}{2}(\sigma_{XY}^2 + \sigma_{XZ}^2 - 2 \sigma_{XYXZ})$ and the variance summary in Equation~\eqref{eq:var_AsymptoticDistributionDropHalf} is $
\frac{1}{2}(2\sigma_{XY}^2 + 2\sigma_{XZ}^2)$ where in both cases the statistic is scaled by $\sqrt{m}$.  We have that the variance of the independent test statistic is smaller than the variance of the dependent test statistic when
\begin{align}
\frac{1}{2} (\sigma_{XY}^2 + \sigma_{XZ}^2 - 2 \sigma_{XYXZ}) &< \frac{1}{2} (2\sigma_{XY}^2 + 2\sigma_{XZ}^2) \nonumber \\
\Longleftrightarrow - 2 \sigma_{XYXZ} &< \sigma_{XY}^2 + \sigma_{XZ}^2
\end{align}
which is implied by the positive definiteness of $\Sigma$.
\end{proof}
\begin{proof}[Proof of Theorem~\ref{thm:powerful_test}]
The Type II error probability of the independent test at level $\alpha$ is
\begin{equation} \label{eq:indepTestPvalue}
  \Phi \left[\Phi^{-1}(1-\alpha) - \dfrac{
   \begin{matrix}  m^{-1/2}\big(  HSIC(\mathcal{F}, \mathcal{G}, P_{xy})  \\ \qquad- HSIC(\mathcal{F}, \mathcal{H}, P_{xz})\big)\end{matrix}}
    {\sqrt{\sigma^2_{X'Y'}+\sigma^2_{X''Z''}}} \right], 
\end{equation}
where we again make the most conservative possible assumption that $HSIC(\mathcal{F}, \mathcal{G}, P_{xy}) - HSIC(\mathcal{F}, \mathcal{H}, P_{xz})=0$ under the null. The Type II error probability of the dependent test at level $\alpha$ is
\begin{equation}
  \label{eq:depTestPvalue}
  \Phi \left[ \Phi^{-1}(1-\alpha) -
\dfrac{
   \begin{matrix}  m^{-1/2}\big(  HSIC(\mathcal{F}, \mathcal{G}, P_{xy})  \\ \qquad- HSIC(\mathcal{F}, \mathcal{H}, P_{xz})\big)\end{matrix}}
    {\sqrt{\sigma^2_{XY}+\sigma^2_{XZ}-2\sigma_{XYXZ}}} \right]
\end{equation}
where $\Phi$ is the CDF of the standard normal distribution. 
\Red{The numerator in Eq.~\eqref{eq:indepTestPvalue} is the same as the numerator in Eq.~\eqref{eq:depTestPvalue}, and the denominator in Eq.~\eqref{eq:depTestPvalue} is smaller due to Lemma~\ref{lem:lowerVariance}.  
  The lower variance dependent test therefore has higher ARE, i.e.,\ for a sufficient sample size $m > \tau$ for some distribution dependent $\tau \in \mathbb{N}_+$, the dependent test will be more powerful than the independent test.
}
\end{proof}


\section{Generalizing to more than two HSIC statistics}\label{gen_depend_test}

The generalization of the dependence test to more than three random variables follows from the earlier derivation  by applying successive rotations to a higher dimensional joint Gaussian distribution over multiple HSIC statistics.  We assume a sample $\mathcal{S}$ of size $m$ over $n$ domains with kernels $k_1, \dots, k_n$ associated uniquely with respective reproducing kernel Hilbert spaces $\mathcal{F}_1, \dots, \mathcal{F}_n$.  We define a generalized statistical test, $\mathcal{T}_g (\mathcal{S}) \rightarrow \lbrace 0,1 \rbrace$ to test the null hypothesis $\mathcal{H}_0$ : $\sum_{(x,y) \in \{1,\dots,n\}^2} v_{(x,y)}  HSIC(\mathcal{F}_x,\mathcal{F}_y,P_{xy}) \leq 0$ versus the alternative hypothesis $\mathcal{H}_m$ : $\sum_{(x,y) \in \{1,\dots,n\}^2} v_{(x,y)}  HSIC(\mathcal{F}_x,\mathcal{F}_y,P_{xy}) > 0$, where $v$ is a vector of weights on each HSIC statistic.  We may recover the test in the previous section by setting $v_{(1,2)}=+1$ $v_{(1,3)}=-1$ and $v_{(i,j)}=0$ for all $(i,j) \in \{1,2,3\}^2\setminus \{(1,2),(1,3)\}$.

The derivation of the test follows the general strategy used in the previous section: we construct a rotation matrix so as to project the joint Gaussian distribution onto the first axis, and read the $p$-value from a standard normal table.  To construct the rotation matrix, we simply need to rotate $v$ such that it is aligned with the first axis.  Such a rotation can be computed by composing $n$ 2-dimensional rotation matrices as in Algorithm~\ref{alg:generalization_algo}.
\begin{algorithm}[h!]
   \caption{Successive rotation for generalized high-dimensional relative tests of dependency (cf.\ Section~\ref{gen_depend_test})}
   \label{alg:generalization_algo}
\begin{algorithmic}
\REQUIRE $v \in \mathbb{R}^n$
\ENSURE $[Q v]_i = 0 \ \ \forall i \ne 1$, $Q^T Q = I$
\STATE $Q=I$
\FOR{$i=2$ {\bfseries to} $n$}
\STATE{$Q_i = I$; $\theta = - \tan^{-1} \frac{v_i}{[Qv]_1}$}
\STATE{$[Q_i]_{11} = \cos(\theta)$; $[Q_i]_{1i} = -\sin(\theta)$}
\STATE{$[Q_i]_{i1} = \sin(\theta)$; $[Q_i]_{ii} = \cos(\theta)$}
\STATE{$Q = Q_i Q$}
\ENDFOR
\end{algorithmic}
\end{algorithm}
%



\section{Experiments}\label{sec:experiments}

We apply our estimates of statistical dependence to three challenging problems. The first is a synthetic data experiment, in which we can directly control the relative degree of functional dependence between variates.  The second experiment uses a multilingual corpus to determine the relative relations between European languages.  The last experiment is a $3$-block dataset which combines gene expression, comparative genomic hybridization, and a qualitative phenotype measured on a sample of Glioma patients.


\subsection{Synthetic experiment}\label{sec:synthetic_experiments}
We constructed 3 distributions as defined in Eq.~\eqref{eq:synthetic_data} and illustrated in Figure \ref{fig:illustration_synt_data}.  
\begin{align}
\label{eq:synthetic_data}
\mbox{Let } t &\sim  \mathcal{U}[(0,2\pi)], \\
   (a)\ x_1 &\sim t + \gamma_1 \mathcal{N}(0,1) \hspace{0.3cm}
   y_1 \sim \sin(t) + \gamma_1 \mathcal{N}(0,1)& \nonumber \\
   (b)\ x_2 &\sim t \cos(t) + \gamma_2 \mathcal{N}(0,1) \hspace{0.3cm}
   y_2 \sim  t \sin(t) + \gamma_2 \mathcal{N}(0,1)& \nonumber \\
   (c)\ x_3 &\sim  t \cos(t) + \gamma_3 \mathcal{N}(0,1) \hspace{0.3cm}
   y_3 \sim  t \sin(t) + \gamma_3 \mathcal{N}(0,1)& \nonumber    
\end{align}

\begin{figure}
\begin{tabular}{rcrcrc}
   \scalebox{0.8}{\tiny \begin{sideways} $ \sin(t) + \gamma_1 \mathcal{N}(0,1)$ \end{sideways}} & \includegraphics[width=0.14\textwidth]{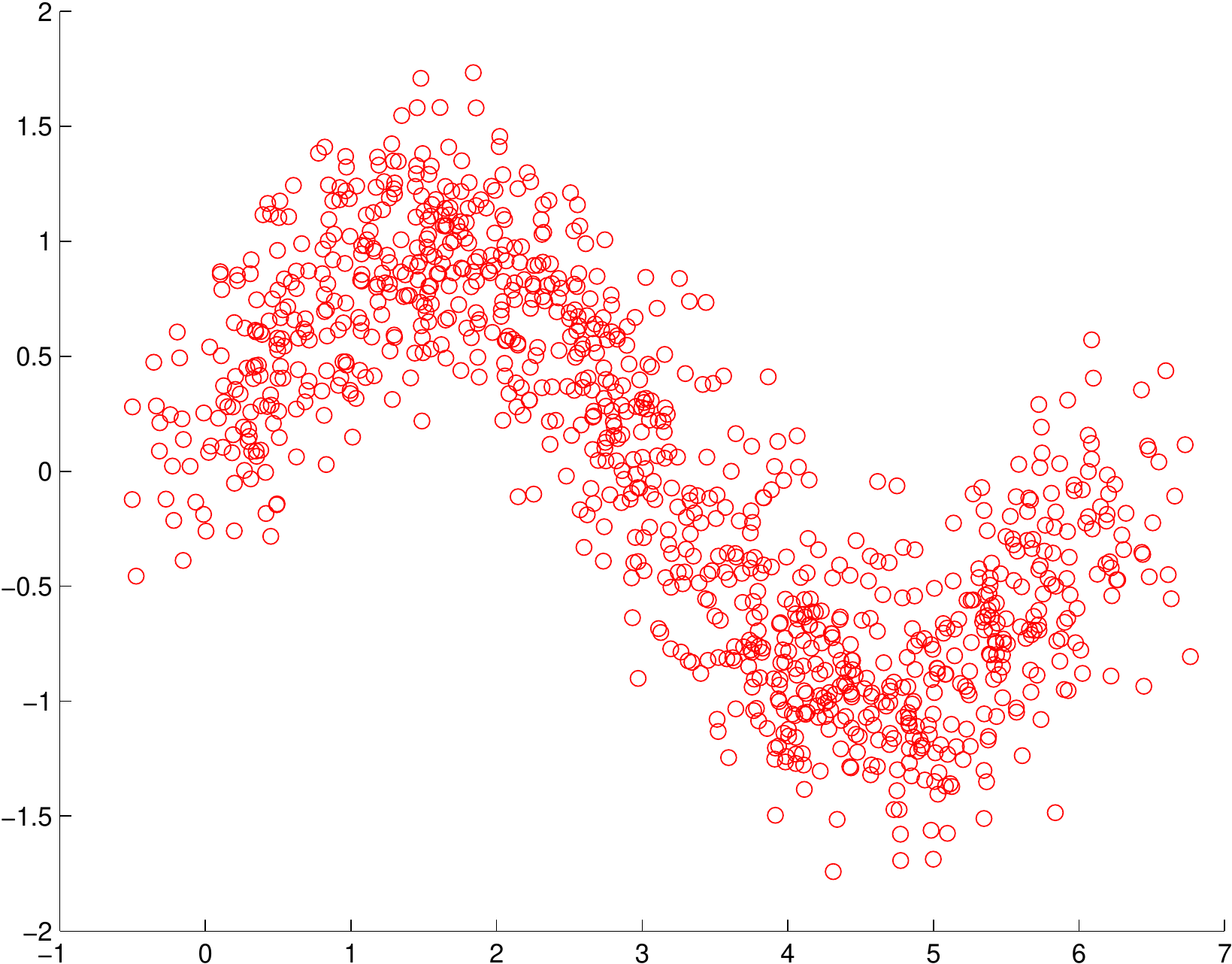} & \scalebox{0.8}{\tiny\begin{sideways}  $ t \sin(t) + \gamma_2 \mathcal{N}(0,1)$ \end{sideways}} &
   \includegraphics[width=0.14\textwidth]{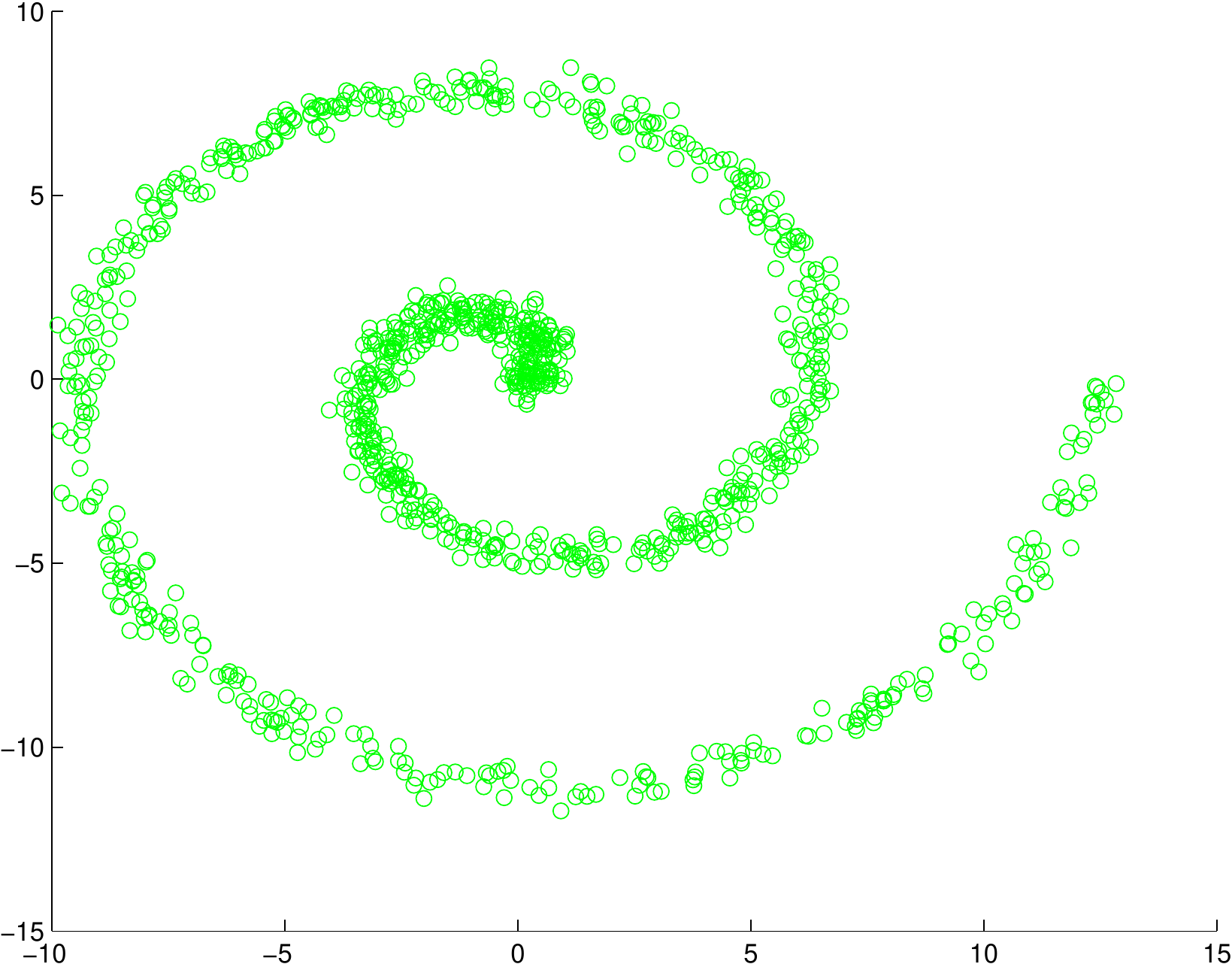} & \scalebox{0.8}{\tiny\begin{sideways}  $ t \cos(t) + \gamma_3 \mathcal{N}(0,1)$ \end{sideways}} &
   \includegraphics[width=0.14\textwidth]{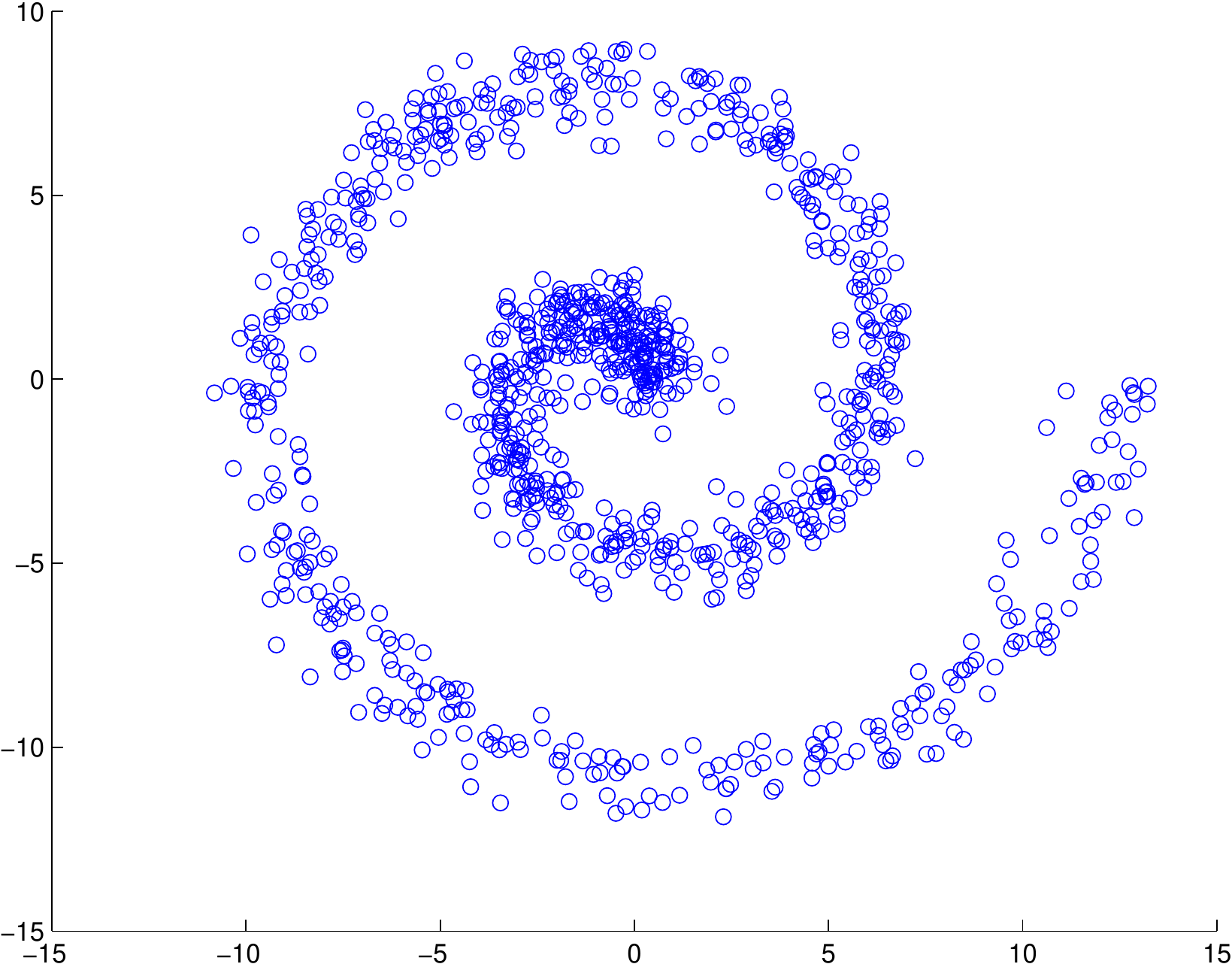} \\
   & \scalebox{0.8}{\tiny $t + \gamma_1 \mathcal{N}(0,1)$}&& \scalebox{0.8}{\tiny $t \cos(t) + \gamma_2 \mathcal{N}(0,1)$} && \scalebox{0.8}{\tiny $t \cos(t) + \gamma_3 \mathcal{N}(0,1)$} \\
  & (a) $\gamma_1 = 0.3$ && (b) $\gamma_2 = 0.3$ && (c) $\gamma_3 = 0.6$ \\
\end{tabular}
\caption{ Illustration of a synthetic dataset sampled from the distribution in Eq.~\eqref{eq:synthetic_data}.}
   \label{fig:illustration_synt_data}
\end{figure}
These distributions are specified so that we can control the relative degree of functional dependence between the variates by varying the relative size of noise scaling parameters $\gamma_1$, $\gamma_2$ and $\gamma_3$. 
The question is then whether the dependence between (a) and (b) is larger than the dependence between (a) and (c). 
In these experiments, we fixed $\gamma_1 = \gamma_2 = 0.3$, while we varied $\gamma_3$, and used a Gaussian kernel with bandwidth $\sigma$ selected as the median pairwise distance between data points.  This kernel is sufficient to obtain good performance, although others choices exist \citep{gretton2012optimal}.  

Figure \ref{fig:powerofthetest} shows the power of the dependent and the independent tests as we vary $\gamma_3$. 
It is clear from these results that the dependent test is far more powerful than the independent test over the great majority of $\gamma_3$ values considered.
 Figure~\ref{fig:empirical_HSIC_differentsamplesize} demonstrates that this superior test power arises due to the tighter and more concentrated distribution
of the dependent statistic.

\begin{figure}[!ht]
\centering
\includegraphics[width=0.35\textwidth]{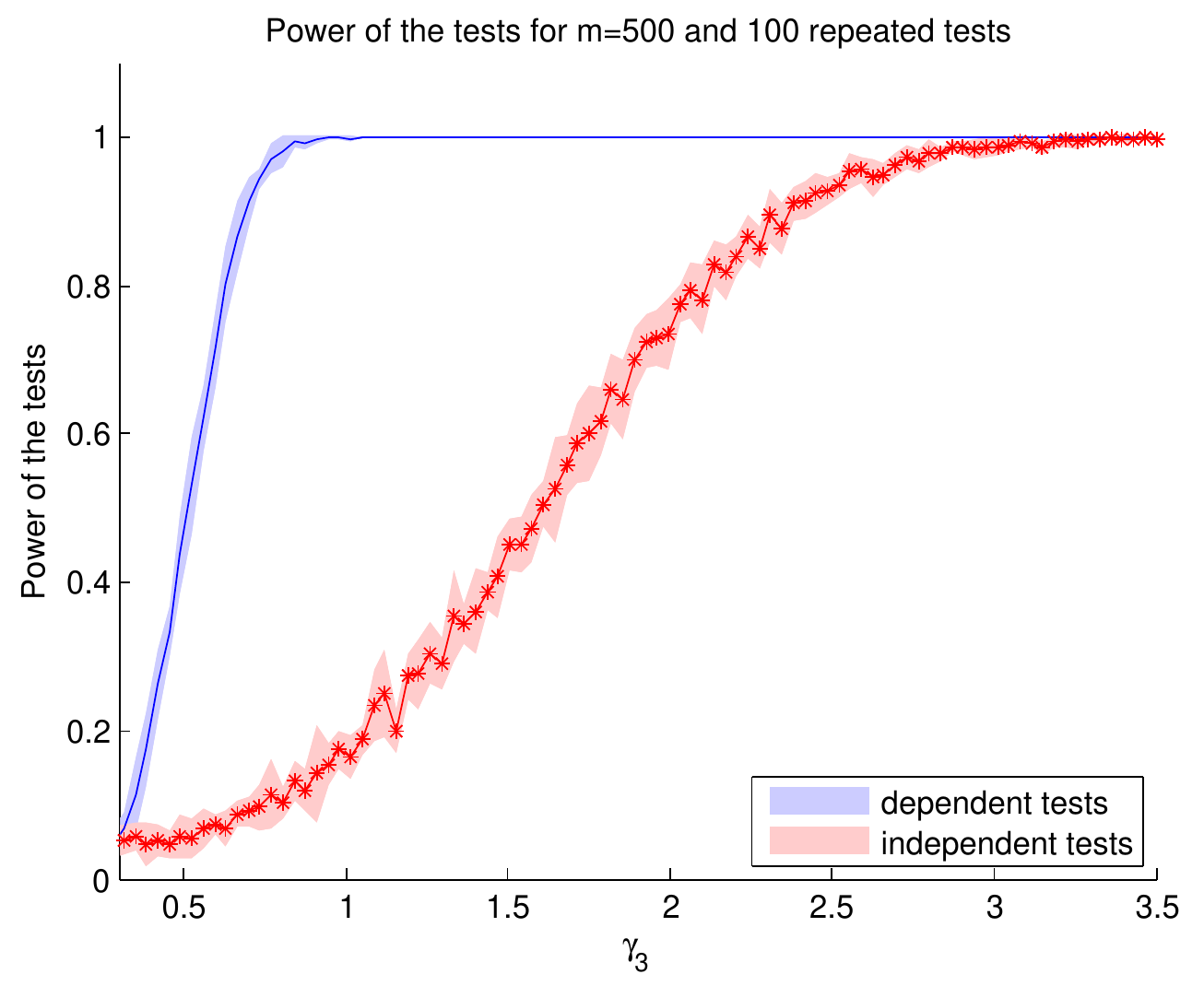}
\caption{Power of the dependent and independent test as a function of $\gamma_3$ on the synthetic data described in Section~\ref{sec:synthetic_experiments}.  \Red{For values of $\gamma_3>0.3$ the distribution in Fig.~\ref{fig:illustration_synt_data}(a) is closer to \ref{fig:illustration_synt_data}(b) than to \ref{fig:illustration_synt_data}(c). The problem becomes difficult as $\gamma_3\rightarrow 0.3$.} As predicted by theory, the dependent test is significantly more powerful over almost all values of $\gamma_3$ by a substantial margin. }
\label{fig:powerofthetest}
\end{figure}	

\begin{figure*}[!ht]
\centering
\begin{tabular}{ccc}
   \includegraphics[width=0.3\textwidth]{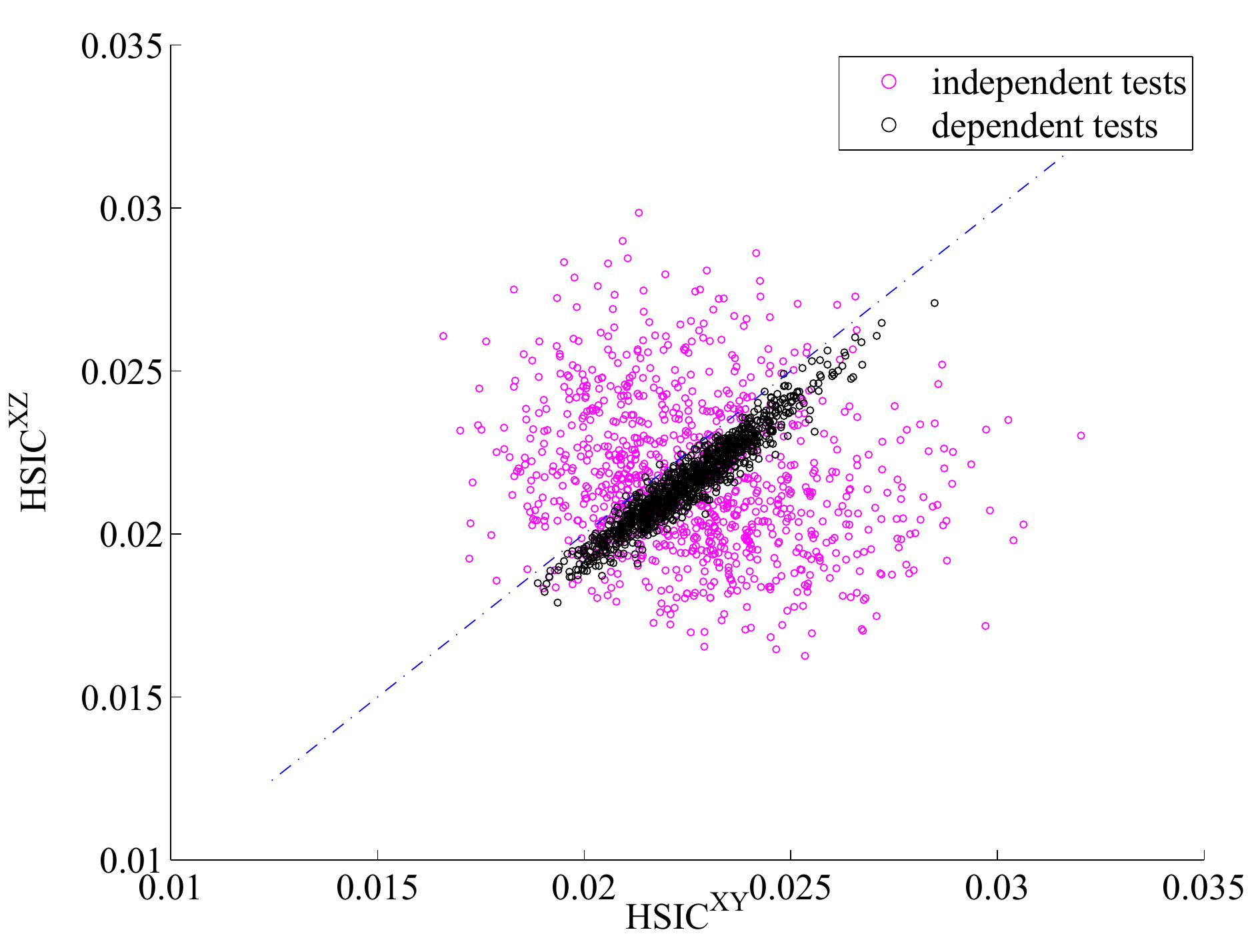} &
   \includegraphics[width=0.3\textwidth]{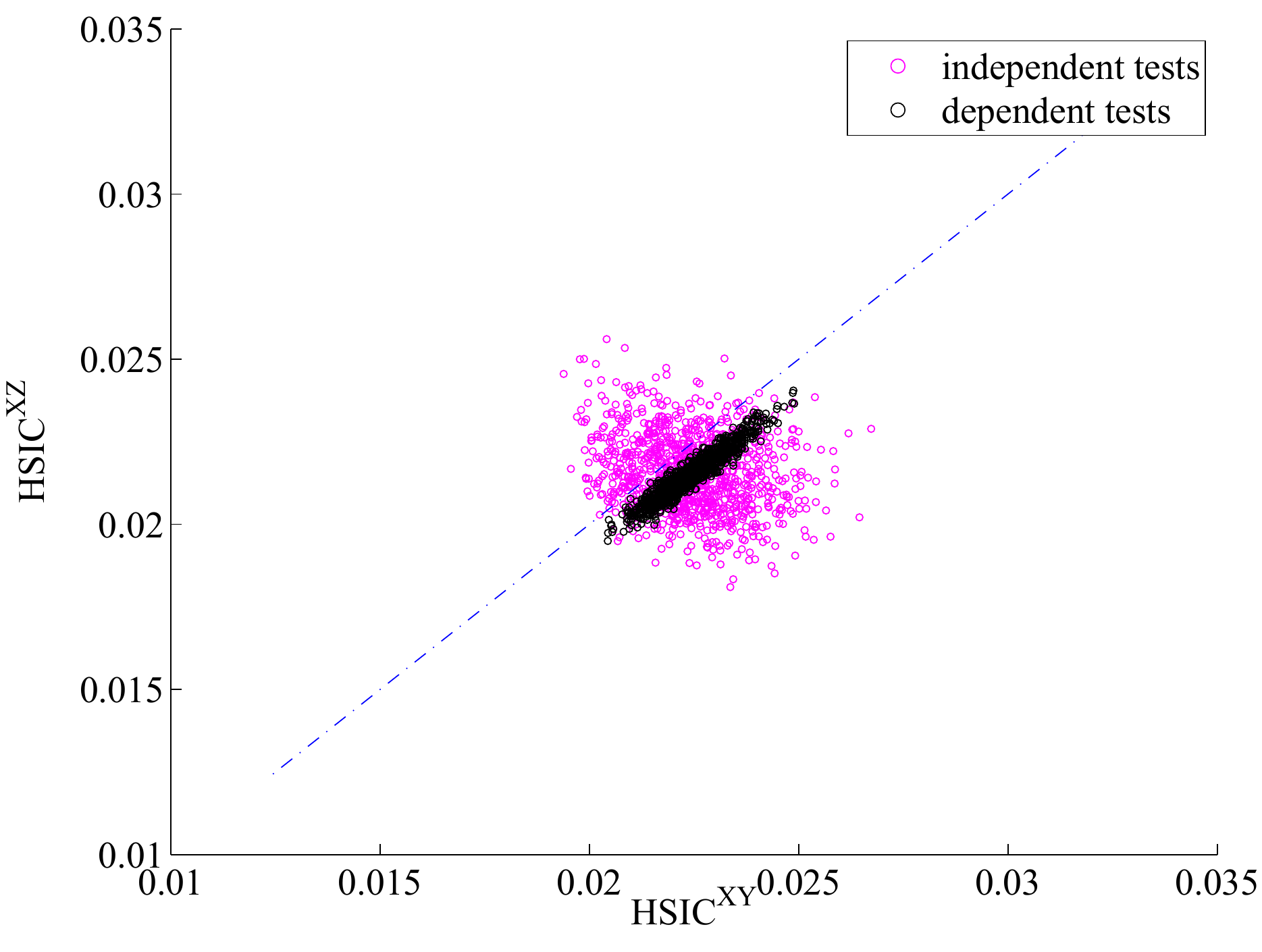} &
      \includegraphics[width=0.3\textwidth]{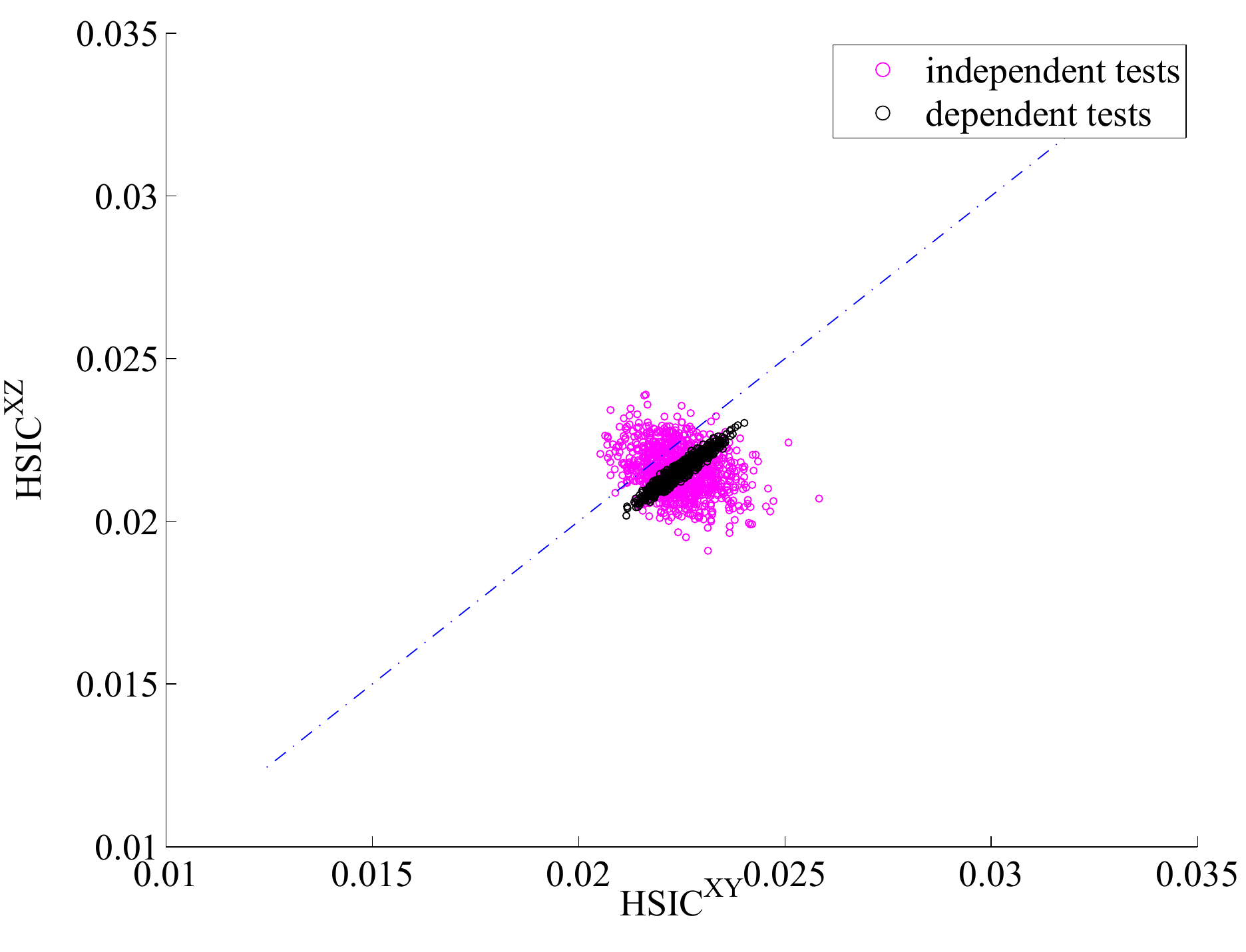} \\
   (a) m=500, $\gamma_3=0.7$ & 
   (b) m=1000, $\gamma_3=0.7$ &   
   (c)  m=3000, $\gamma_3=0.7$\\
    $\textrm{p}_{\textrm{dep}} = 0.0189$, $\textrm{p}_{\textrm{indep}} = 0.3492$& 
    $\textrm{p}_{\textrm{dep}} = 10^{-4}$, 
    $\textrm{p}_{\textrm{indep}} = 0.3690$ & 
    $\textrm{p}_{\textrm{dep}} = 10^{-6}$, 
    $\textrm{p}_{\textrm{indep}} = 0.2876$ \\
    \\
   \includegraphics[width=0.3\textwidth]{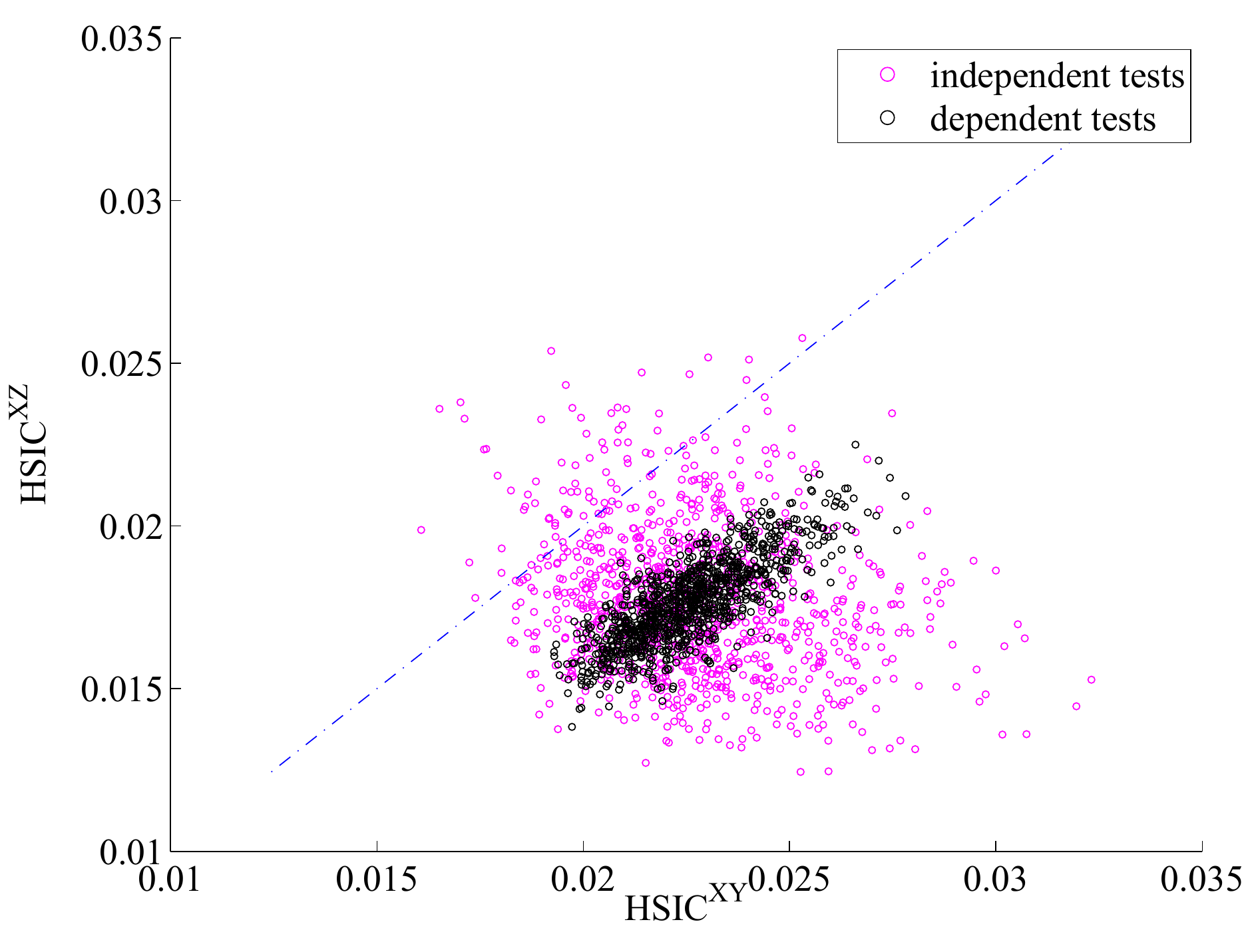} &
   \includegraphics[width=0.3\textwidth]{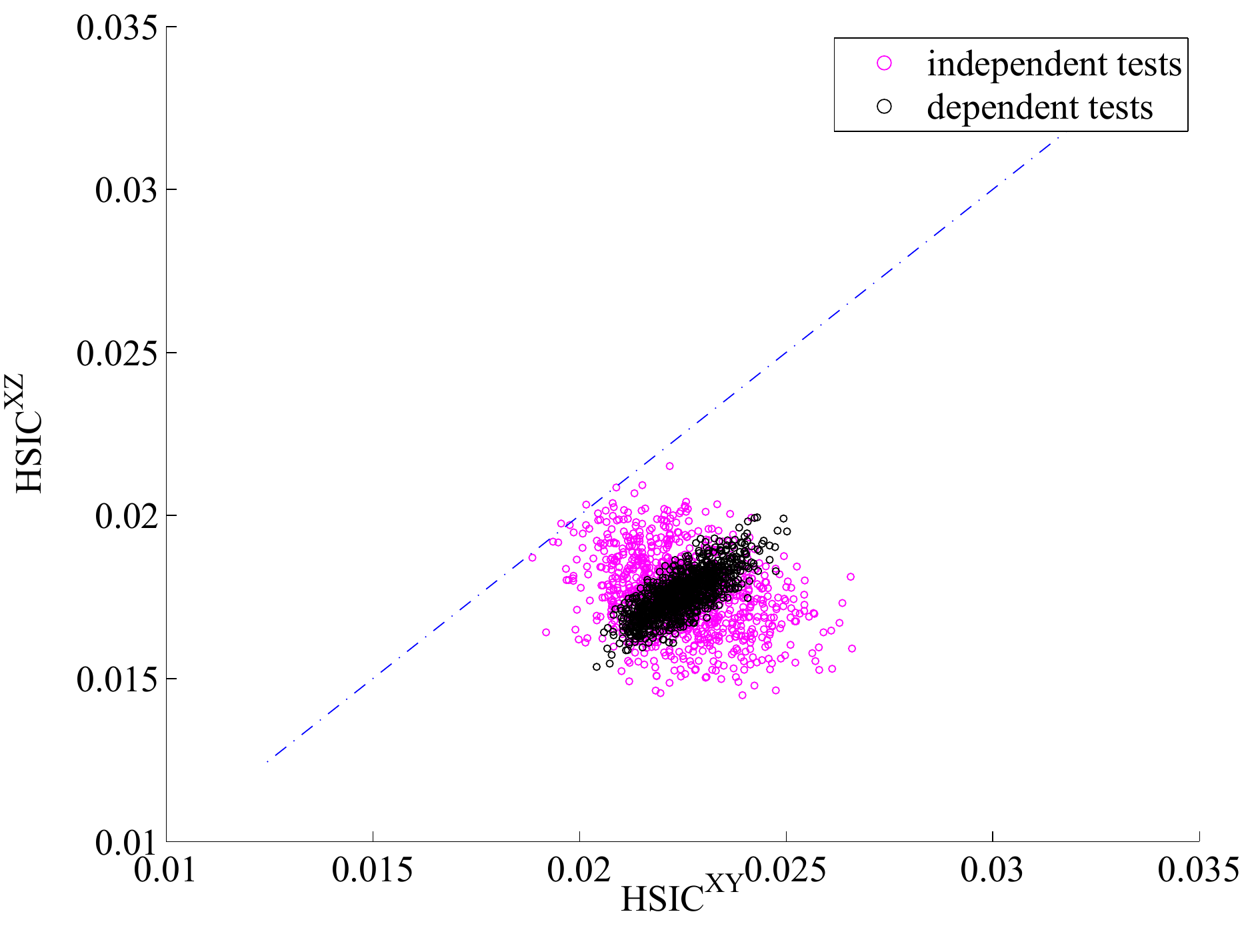} &
   \includegraphics[width=0.3\textwidth]{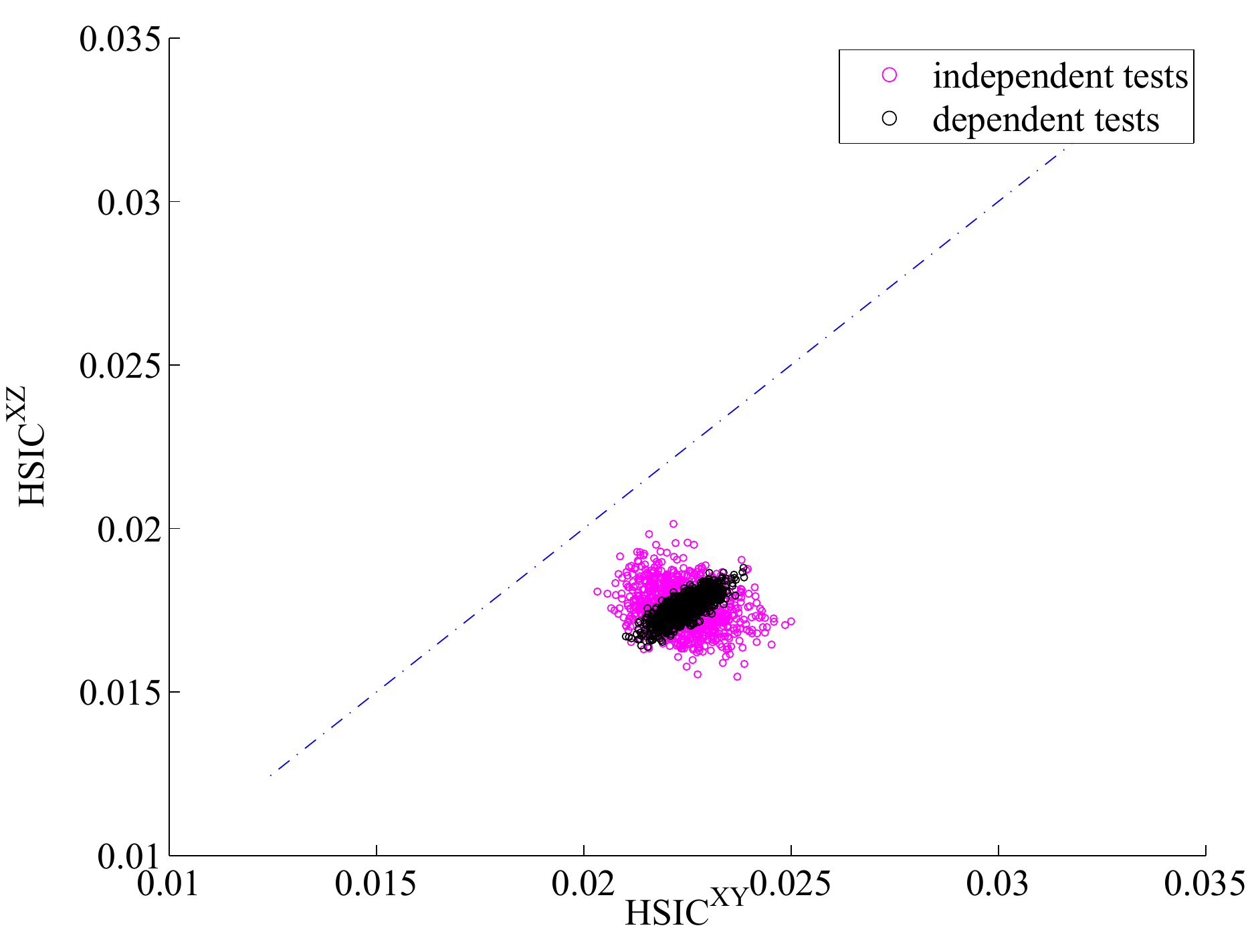} \\
   (d) m=500, $\gamma_3=1.7$ & (e) m=1000, $\gamma_3=1.7$ & (f) m=3000, $\gamma_3=1.7$\\
   $\textrm{p}_{\textrm{dep}} = 10^{-9}$, $\textrm{p}_{\textrm{indep}} = 0.982$  & $\textrm{p}_{\textrm{dep}} = 10^{-10}$, $\textrm{p}_{\textrm{indep}} = 0.0326$ & $\textrm{p}_{\textrm{dep}} = 10^{-13}$, 
   $\textrm{p}_{\textrm{indep}} = 0.005$ \\
\end{tabular}
\caption{For the synthetic experiments described in Section~\ref{sec:synthetic_experiments}, we plot empirical HSIC values for dependent and independent tests for 100 repeated draws with different sample sizes.  Empirical $p$-values for each test show that the dependent distribution converges faster than the independent distribution even at low sample size, resulting in a more powerful statistical test.}
   \label{fig:empirical_HSIC_differentsamplesize}
\end{figure*}
%


\subsection{Multilingual data}\label{translate}
In this section,  we demonstrate dependence testing to predict the relative similarity of different languages.  We use a real world dataset taken from the parallel European Parliament corpus~\cite{koehn2005europarl}. 
We choose 3000 random documents in common written in: Finnish (fi), Italian (it), French (fr), Spanish (es), Portuguese (pt), English (en), Dutch (nl), German (de), Danish (da) and Swedish (sv).  These languages can be broadly categorized into either the Romance, Germanic or Uralic groups~\cite{gray2003language}.  In this dataset, we considered each language as a random variable and each document as an observation. 

Our first goal is to test if the statistical dependence between two languages in the same group is greater than the statistical dependence between languages in different groups.  
For pre-processing, we removed stop-words (\url{http://www.nltk.org}) and performed stemming (\url{http://snowball.tartarus.org}).  We applied the TF-IDF model as a feature representation and used a Gaussian kernel with the bandwidth $\sigma$ set per language as the median pairwise distance between documents. 

In Table~\ref{tab:europeanlang_pvalue}, a selection of tests between language groups (Germanic, Romance, and Uralic) is given:  all $p$-values strongly support that our relative dependence test finds the different language groups with very high significance.
\begin{table}[ht]
\centering
\begin{tabular}{cccc}
\hline  Source & Target 1 & Target 2 & $p$-value \\
\hline
 es & pt & fi & $\mathbf{0.0066}$ \\
 fr & it & da & $\mathbf{0.0418}$ \\
 it & es & fi & $\mathbf{0.0169}$ \\
 pt & es & da & $\mathbf{0.0173}$ \\
 de & nl & fi & $\mathbf{<10^{-4}}$ \\
 nl & en & es & $\mathbf{<10^{-4}}$ \\
 da & sv & fr & $\mathbf{<10^{-6}}$ \\
 sv & en & it & $\mathbf{<10^{-4}}$ \\
 en & de & es & $\mathbf{<10^{-4}}$ \\
\hline
\end{tabular}
\caption{A selection of relative dependency tests between two pairs of HSIC statistics for the multilingual corpus data. \Red{Low $p$-values indicate a source is closer to target 1 than to target 2.}
In all cases, the test correctly identifies that languages within the same group are more strongly related than those in different groups.}
\label{tab:europeanlang_pvalue}
\end{table}

Further, if we focus on the Romance family, our test enables one to answer more fine-grained questions about the relative similarity of languages within the same group.  As before, we determine the ground truth similarities from the topology of the tree of European languages determined by the linguistics community~\cite{gray2003language,Bouckaert2012} as illustrated in Fig.~\ref{fig:romance_tree} for the Romance group.  We have run the test on all triplets from the corpus for which the topology of the tree specifies a correct ordering of the dependencies. In a fraction of a second (excluding kernel computation), we are able to recover certain features of the subtree of relationships between languages present in the Romance language group  (Table~\ref{tab:romance_lang}).  
The test always indicates the correct relative similarity of languages when nearby languages (pt,es) are compared with those further away (ft,it), however errors are made when comparing triplets of languages \Red{for which the nearest common ancestor is more than one link removed}.
\begin{figure}[!ht]
\centering
\includegraphics[width=0.3\textwidth]{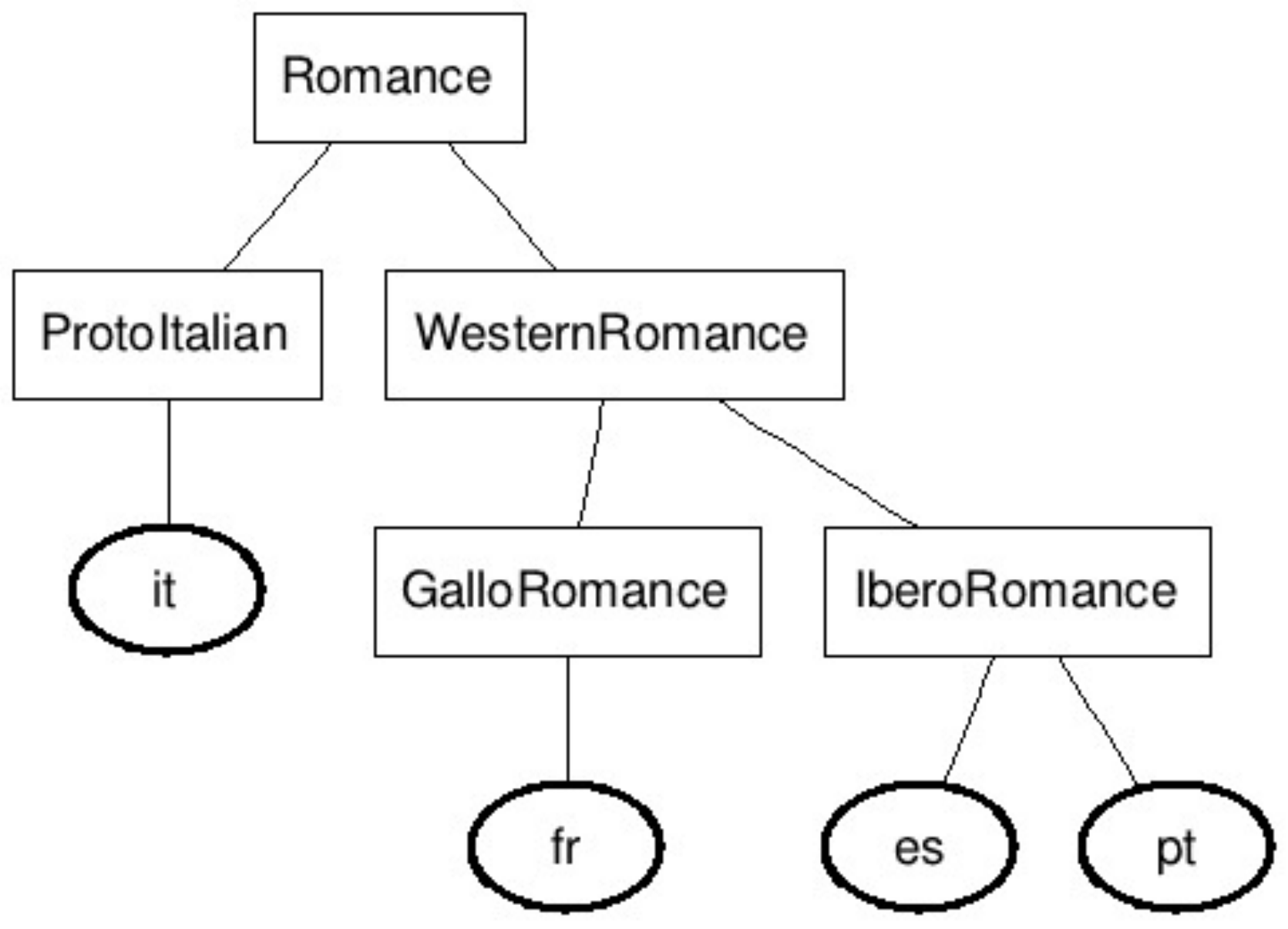}
\caption{Partial tree of Romance languages adapted from~\cite{gray2003language}.}
\label{fig:romance_tree}
\end{figure}
\begin{table}[!ht]
\centering
\begin{tabular}{cccc}
\hline  Source & Target 1 & Target 2 & $p$-value \\
\hline
fr & es & it & $\mathbf{0.0157}$ \\
fr & pt & it & $0.1882$ \\
es & fr & it & $0.2147$ \\
es & pt & it & $\mathbf{<10^{-4}}$ \\
es & pt & fr & $\mathbf{<10^{-4}}$ \\
pt & fr & it & $0.7649$ \\
pt & es & it & $\mathbf{ 0.0011}$ \\
pt & es & fr & $\mathbf{<10^{-8}}$ \\
\hline
\end{tabular}
\caption{Relative dependency tests between Romance languages. The tests are ordered such that a low $p$-value corresponds with a confirmation of the topology of the tree of Romance languages determined by the linguistics community \cite{gray2003language}.}
\label{tab:romance_lang}
\end{table}

In our next tests, we evaluate our more general framework for testing relative dependencies with more than two HSIC statistics.  We chose four languages, and tested whether the average dependence between languages in the same group is higher than the dependence between groups.  The results of these tests are in Table~\ref{tab:europeanlang_pvalue2}.  As before, our   
test is able to distinguish language groups with high significance.
\begin{table}[!ht]
\centering
\begin{tabular}{ccc}
\hline  Source & Targets & $p$-value \\
\hline
da & de  sv  fi & $\mathbf{<10^{-9}}$ \\
da & sv  en  fr & $\mathbf{<10^{-9}}$ \\
de & sv  en  it & $\mathbf{<10^{-5}}$ \\
fr & it  es  sv & $\mathbf{<10^{-5}}$ \\
es & fr  pt  nl & $\mathbf{0.0175}$ \\
\hline
\end{tabular}
\caption{Relative dependency test between four pairs of HSIC statistics for the multilingual corpus data.  These tests show the ability of the relative dependence test to generalize to arbitrary numbers of HSIC statistics by constructing a rotation matrix using Algorithm~\ref{alg:generalization_algo}. \Red{In all cases $v=[1$  $1$ $-2]$.}}
\label{tab:europeanlang_pvalue2}
\end{table}

%
\subsection{Pediatric glioma data}\label{sec:neuroscience_data}

Brain tumors are the most common solid tumors in children and have the highest mortality rate of all pediatric cancers. Despite advances in multimodality therapy, children with pediatric high-grade gliomas (pHGG) invariably have an overall survival of around 20\% at 5 years. Depending on their location (e.g.\ brainstem, central nuclei, or supratentorial), pHGG present different characteristics in terms of radiological appearance, histology, and prognosis. The hypothesis is that pHGG have different genetic origins and oncogenic pathways depending on their location. Thus, the biological processes involved in the development of the tumor may be different from one location to another.

In order to evaluate such hypotheses, pre-treatment frozen tumor samples were obtained from 53 children with newly diagnosed pHGG from Necker Enfants Malades (Paris, France) from Puget et al, \yrcite{puget2012mesenchymal}.  The 53 tumors are divided into 3 locations: supratentorial (HEMI), central nuclei (MIDL), and brain stem (DIPG).  The final dataset is organized in 3 blocks of variables defined for the 53 tumors: X is a block of indicator variables describing the location category, the second data matrix Y provides the expression of 15 702 genes (GE). The third data matrix Z contains the imbalances of 1229 segments (CGH) of chromosomes.  

For X, we use a linear kernel, which is characteristic for indicator variables, and for Y and Z, the kernel was chosen to be the Gaussian kernel with $\sigma$ selected as the median of pairwise distances.  The $p$-value of our relative dependency test is $<10^{-5}$. This shows that the tumor location in the brain is more dependent on gene expression than on chromosomal imbalances.  By contrast with Section~\ref{sec:synthetic_experiments}, the independent test was also able to find the same ordering of dependence, but with a $p$-value that is three orders of magnitude larger ($p=0.005$).  Figure~\ref{fig:glioma_twosigma} shows iso-curves of the Gaussian distributions estimated in the independent and dependent tests. The empirical relative dependency is consistent with findings in the medical literature, and provides additional statistical support for the importance of tumor location in Glioma~\cite{gilbertson2007,Palm2009,puget2012mesenchymal}.
\begin{figure}
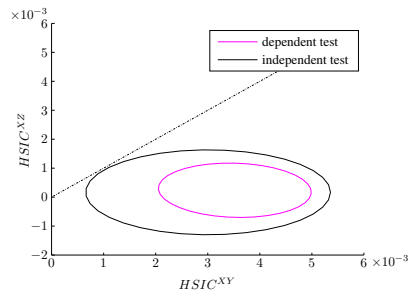

\centering
\scalebox{.5} 
{
\begin{pgfpicture}
  \begin{pgfscope}
    \definecolor{matfig2pgf_color}{rgb}{1,1,1}\pgfsetfillcolor{matfig2pgf_color}
    \pgfpathrectangle{\pgfpoint{1.41099cm}{1.05901cm}}{\pgfpoint{8.30594cm}{6.1634cm}}
    \pgfusepath{fill}
  \end{pgfscope}
  \begin{pgfscope}
    \pgfsetlinewidth{0.5pt}
    \definecolor{matfig2pgf_linecolor}{rgb}{0.000,0.000,0.000}
    \pgfsetstrokecolor{matfig2pgf_linecolor}
    \foreach \x in {1.41099,2.79531,4.17963,5.56396,6.94828,8.3326,9.71692}
    {
      \pgfpathmoveto{\pgfpoint{\x cm}{1.05901cm}}\pgfpathlineto{\pgfpoint{\x cm}{1.12065cm}}
    }
    \pgfusepath{stroke}
  \end{pgfscope}
  \begin{pgfscope}
    \pgfsetlinewidth{0.5pt}
    \definecolor{matfig2pgf_linecolor}{rgb}{0.000,0.000,0.000}
    \pgfsetstrokecolor{matfig2pgf_linecolor}
    \foreach \y in {1.05901,1.82944,2.59986,3.37029,4.14071,4.91113,5.68156,6.45198,7.22241}
    {
      \pgfpathmoveto{\pgfpoint{1.41099cm}{\y cm}}\pgfpathlineto{\pgfpoint{1.49405cm}{\y cm}}
    }
    \pgfusepath{stroke}
  \end{pgfscope}
  \begin{pgfscope}
    \pgfsetlinewidth{0.5pt}
    \definecolor{matfig2pgf_linecolor}{rgb}{0.000,0.000,0.000}
    \pgfsetstrokecolor{matfig2pgf_linecolor}
    \pgfpathmoveto{\pgfpoint{1.410987cm}{1.059012cm}}\pgfpathlineto{\pgfpoint{9.716923cm}{1.059012cm}}
    \pgfusepath{stroke}
  \end{pgfscope}
  \begin{pgfscope}
    \pgfsetlinewidth{0.5pt}
    \definecolor{matfig2pgf_linecolor}{rgb}{0.000,0.000,0.000}
    \pgfsetstrokecolor{matfig2pgf_linecolor}
    \pgfpathmoveto{\pgfpoint{1.410987cm}{1.059012cm}}\pgfpathlineto{\pgfpoint{1.410987cm}{7.222407cm}}
    \pgfusepath{stroke}
  \end{pgfscope}
  \begin{pgfscope}
    \definecolor{matfig2pgf_fillcolor}{rgb}{0.000,0.000,0.000}
    \pgfsetfillcolor{matfig2pgf_fillcolor}
    \pgftext[x=1.41099cm,y=0.959012cm,top]{$0$}
    \pgftext[x=2.79531cm,y=0.959012cm,top]{$1$}
    \pgftext[x=4.17963cm,y=0.959012cm,top]{$2$}
    \pgftext[x=5.56396cm,y=0.959012cm,top]{$3$}
    \pgftext[x=6.94828cm,y=0.959012cm,top]{$4$}
    \pgftext[x=8.3326cm,y=0.959012cm,top]{$5$}
    \pgftext[x=9.71692cm,y=0.959012cm,top]{$6$}
    \pgftext[x=9.91692cm,y=1.05901cm,left,top]{$\times 10^{-3}$}
  \end{pgfscope}
  \begin{pgfscope}
    \definecolor{matfig2pgf_fillcolor}{rgb}{0.000,0.000,0.000}
    \pgfsetfillcolor{matfig2pgf_fillcolor}
    \pgftext[x=1.31099cm,y=1.05901cm,right]{$-2$}
    \pgftext[x=1.31099cm,y=1.82944cm,right]{$-1$}
    \pgftext[x=1.31099cm,y=2.59986cm,right]{$0$}
    \pgftext[x=1.31099cm,y=3.37029cm,right]{$1$}
    \pgftext[x=1.31099cm,y=4.14071cm,right]{$2$}
    \pgftext[x=1.31099cm,y=4.91113cm,right]{$3$}
    \pgftext[x=1.31099cm,y=5.68156cm,right]{$4$}
    \pgftext[x=1.31099cm,y=6.45198cm,right]{$5$}
    \pgftext[x=1.31099cm,y=7.22241cm,right]{$6$}
   \pgftext[x=1.31099cm,y=7.32241cm,bottom,right]{$\times 10^{-3}$}
  \end{pgfscope}
  \makeatletter\ifpgf@draftmode\makeatother\else
  \begin{pgfscope}
    \pgfpathrectangle{\pgfpoint{1.41099cm}{1.05901cm}}{\pgfpoint{8.30594cm}{6.1634cm}}
    \pgfusepath{clip}
    \begin{pgfscope}
      \pgfsetlinewidth{0.50pt}
      \definecolor{matfig2pgf_linecolor}{rgb}{1.000,0.000,1.000}
      \pgfsetstrokecolor{matfig2pgf_linecolor}
      \pgfsetdash{}{0pt}
      \pgfsetroundjoin
      \pgfplothandlerlineto
\pgfplotstreamstart
\foreach \x/\y in {8.315/2.697,8.300/2.827,8.220/2.955,8.078/3.078,7.878/3.191,7.627/3.291,7.332/3.375,7.004/3.440,6.653/3.483
,6.290/3.504,5.927/3.502,5.575/3.477,5.246/3.430,4.951/3.361,4.698/3.274,4.496/3.172,4.352/3.056,4.270/2.932,4.252/2.804
,4.300/2.674,4.412/2.548,4.583/2.430,4.810/2.323,5.083/2.230,5.396/2.156,5.737/2.101,6.095/2.069,6.460/2.059,6.819/2.073
,7.160/2.110,7.474/2.168,7.749/2.246,7.977/2.341,8.151/2.450,8.265/2.570,8.315/2.697}
{
\pgfplotstreampoint{\pgfpoint{\x cm}{\y cm}}
}
\pgfplotstreamend
      \pgfusepath{stroke}
    \end{pgfscope}
  \end{pgfscope}
  \fi
  \makeatletter\ifpgf@draftmode\makeatother\else
  \begin{pgfscope}
    \pgfpathrectangle{\pgfpoint{1.41099cm}{1.05901cm}}{\pgfpoint{8.30594cm}{6.1634cm}}
    \pgfusepath{clip}
    \begin{pgfscope}
      \pgfsetlinewidth{0.50pt}
      \definecolor{matfig2pgf_linecolor}{rgb}{0.000,0.000,0.000}
      \pgfsetstrokecolor{matfig2pgf_linecolor}
      \pgfsetdash{}{0pt}
      \pgfsetroundjoin
      \pgfplothandlerlineto
\pgfplotstreamstart
\foreach \x/\y in {8.827/2.729,8.775/2.930,8.619/3.126,8.366/3.308,8.023/3.472,7.602/3.612,7.115/3.723,6.578/3.803,6.009/3.848
,5.427/3.857,4.848/3.830,4.294/3.767,3.780/3.671,3.324/3.545,2.940/3.393,2.641/3.219,2.436/3.029,2.331/2.830,2.331/2.627
,2.436/2.428,2.641/2.238,2.940/2.065,3.324/1.912,3.780/1.786,4.294/1.690,4.848/1.627,5.427/1.600,6.009/1.609,6.578/1.654
,7.115/1.734,7.602/1.845,8.023/1.985,8.366/2.149,8.619/2.332,8.775/2.527,8.827/2.729}
{
\pgfplotstreampoint{\pgfpoint{\x cm}{\y cm}}
}
\pgfplotstreamend
      \pgfusepath{stroke}
    \end{pgfscope}
  \end{pgfscope}
  \fi
    \definecolor{matfig2pgf_fillcolor}{rgb}{0.000,0.000,0.000}
    \pgfsetfillcolor{matfig2pgf_fillcolor}
    \pgftext[top,x=5.55303cm,y=0.457972cm,rotate=0]{$HSIC^{XY}$}
    \definecolor{matfig2pgf_fillcolor}{rgb}{0.000,0.000,0.000}
    \pgfsetfillcolor{matfig2pgf_fillcolor}
    \pgftext[bottom,x=0.8099cm,y=4.10793cm,rotate=90]{$HSIC^{XZ}$}
  \makeatletter\ifpgf@draftmode\makeatother\else
  \begin{pgfscope}
    \pgfpathrectangle{\pgfpoint{1.41099cm}{1.05901cm}}{\pgfpoint{8.30594cm}{6.1634cm}}
    \pgfusepath{clip}
    \begin{pgfscope}
      \pgfsetlinewidth{0.50pt}
      \definecolor{matfig2pgf_linecolor}{rgb}{0.000,0.000,0.000}
      \pgfsetstrokecolor{matfig2pgf_linecolor}
      \pgfsetdash{{0.50pt}{1.00pt}{1.50pt}{1.00pt}}{0pt}
      \pgfsetroundjoin
      \pgfplothandlerlineto
\pgfplotstreamstart
\foreach \x/\y in {1.411/2.600,8.827/6.727}
{
\pgfplotstreampoint{\pgfpoint{\x cm}{\y cm}}
}
\pgfplotstreamend
      \pgfusepath{stroke}
    \end{pgfscope}
  \end{pgfscope}
  \fi
  \begin{pgfscope}
    \definecolor{matfig2pgf_color}{rgb}{1,1,1}\pgfsetfillcolor{matfig2pgf_color}
    \pgfpathrectangle{\pgfpoint{5.61828cm}{5.95559cm}}{\pgfpoint{3.95297cm}{1.07759cm}}
    \pgfusepath{fill}
  \end{pgfscope}
  \begin{pgfscope}
    \pgfsetlinewidth{0.5pt}
    \definecolor{matfig2pgf_linecolor}{rgb}{0.000,0.000,0.000}
    \pgfsetstrokecolor{matfig2pgf_linecolor}
    \pgfusepath{stroke}
  \end{pgfscope}
  \begin{pgfscope}
    \pgfsetlinewidth{0.5pt}
    \definecolor{matfig2pgf_linecolor}{rgb}{0.000,0.000,0.000}
    \pgfsetstrokecolor{matfig2pgf_linecolor}
    \pgfusepath{stroke}
  \end{pgfscope}
  \begin{pgfscope}
    \pgfsetlinewidth{0.5pt}
    \pgfpathrectangle{\pgfpoint{5.61828cm}{5.95559cm}}{\pgfpoint{3.95297cm}{1.07759cm}}
    \pgfusepath{stroke}
  \end{pgfscope}
    \definecolor{matfig2pgf_fillcolor}{rgb}{0.000,0.000,0.000}
    \pgfsetfillcolor{matfig2pgf_fillcolor}
    \pgftext[left,x=7.0128cm,y=6.7143cm,rotate=0]{dependent test}
  \makeatletter\ifpgf@draftmode\makeatother\else
  \begin{pgfscope}
    \pgfpathrectangle{\pgfpoint{5.61828cm}{5.95559cm}}{\pgfpoint{3.95297cm}{1.07759cm}}
    \pgfusepath{clip}
    \begin{pgfscope}
      \pgfsetlinewidth{0.50pt}
      \definecolor{matfig2pgf_linecolor}{rgb}{1.000,0.000,1.000}
      \pgfsetstrokecolor{matfig2pgf_linecolor}
      \pgfsetdash{}{0pt}
      \pgfsetroundjoin
      \pgfplothandlerlineto
\pgfplotstreamstart
\foreach \x/\y in {5.830/6.733,6.887/6.733}
{
\pgfplotstreampoint{\pgfpoint{\x cm}{\y cm}}
}
\pgfplotstreamend
      \pgfusepath{stroke}
    \end{pgfscope}
  \end{pgfscope}
  \fi
  \makeatletter\ifpgf@draftmode\makeatother\else
  \begin{pgfscope}
    \pgfpathrectangle{\pgfpoint{5.61828cm}{5.95559cm}}{\pgfpoint{3.95297cm}{1.07759cm}}
    \pgfusepath{clip}
  \end{pgfscope}
  \fi
    \definecolor{matfig2pgf_fillcolor}{rgb}{0.000,0.000,0.000}
    \pgfsetfillcolor{matfig2pgf_fillcolor}
    \pgftext[left,x=7.0128cm,y=6.23049cm,rotate=0]{independent test}
  \makeatletter\ifpgf@draftmode\makeatother\else
  \begin{pgfscope}
    \pgfpathrectangle{\pgfpoint{5.61828cm}{5.95559cm}}{\pgfpoint{3.95297cm}{1.07759cm}}
    \pgfusepath{clip}
    \begin{pgfscope}
      \pgfsetlinewidth{0.50pt}
      \definecolor{matfig2pgf_linecolor}{rgb}{0.000,0.000,0.000}
      \pgfsetstrokecolor{matfig2pgf_linecolor}
      \pgfsetdash{}{0pt}
      \pgfsetroundjoin
      \pgfplothandlerlineto
\pgfplotstreamstart
\foreach \x/\y in {5.830/6.256,6.887/6.256}
{
\pgfplotstreampoint{\pgfpoint{\x cm}{\y cm}}
}
\pgfplotstreamend
      \pgfusepath{stroke}
    \end{pgfscope}
  \end{pgfscope}
  \fi
  \makeatletter\ifpgf@draftmode\makeatother\else
  \begin{pgfscope}
    \pgfpathrectangle{\pgfpoint{5.61828cm}{5.95559cm}}{\pgfpoint{3.95297cm}{1.07759cm}}
    \pgfusepath{clip}
  \end{pgfscope}
  \fi
  \makeatletter\ifpgf@draftmode\makeatother\pgftext[x=4.99956cm,y=3.75241cm]{\Huge{DRAFT}}\fi
\end{pgfpicture}}
\caption{$2\sigma$ iso-curves of the Gaussian distributions estimated from the pediatric Glioma data.  As before, the dependent test has a much lower variance than the independent test.  The tests support the stronger dependence on the tumor location to gene expression than chromosomal imbalances.}
\label{fig:glioma_twosigma}
\end{figure}	 
%


\section{Conclusions}

We have described a novel statistical test that determines whether a source random variable is more strongly dependent on one target random variable or another.  This test, built on the Hilbert-Schmidt Independence Criterion, is low variance, consistent, and unbiased.  We have shown that our test is strictly more powerful than a test that does not exploit the covariance between HSIC statistics, and empirically achieves $p$-values several orders of magnitude smaller.
We have empirically demonstrated the test performance on synthetic data, where the degree of dependence could be controlled; on the challenging problem of identifying language groups from a multilingual corpus; and for finding the most important determinant of Glioma type.  The computation and memory requirements of the test are quadratic in the sample size, matching the performance of HSIC and related tests for dependence between two random variables.  The test is therefore scalable to the wide range of problem instances where non-parametric dependency tests are currently applied.  We have generalized the test framework to more than two HSIC statistics, and have given an algorithm to construct a consistent, low-variance, unbiased test in this setting.

\subsection*{Acknowledgements}
We thank Ioannis Antonoglou for helpful discussions. 
The first author is supported by a fellowship from CentraleSup\'{e}lec.  This work is partially funded by the European Commission through ERC Grant 259112 and FP7-MCCIG334380.

\bibliography{bibliography}
\bibliographystyle{icml2015}

\end{document}